\documentclass{article} %
\usepackage[preprint]{goodfire}

\usepackage{microtype}
\usepackage{hyperref}
\usepackage{url}
\usepackage{booktabs}
\usepackage{amsmath}
\usepackage{macros}
\usepackage{lineno}
\usepackage[color=pink!20,textsize=scriptsize]{todonotes}

\newcommand{\blocksize}{b}
\newcommand{\topk}{k}

\hypersetup{colorlinks=true, citecolor=primary, linkcolor=primary, urlcolor=primary}

\providecommand{\aut}[1]{\textbf{#1}}
\providecommand{\af}[1]{{\small #1}}

\providecommand{\afn}[1]{\textcolor{antique}{$^{#1}$}}

\newcommand{\costar}{{\color{burgundy}\boldsymbol{\star}}}
\newcommand{\colead}{{\color{burgundy}\boldsymbol{\dagger}}}

\newcommand{\goodfireaff}{%
  \includegraphics[height=18pt]{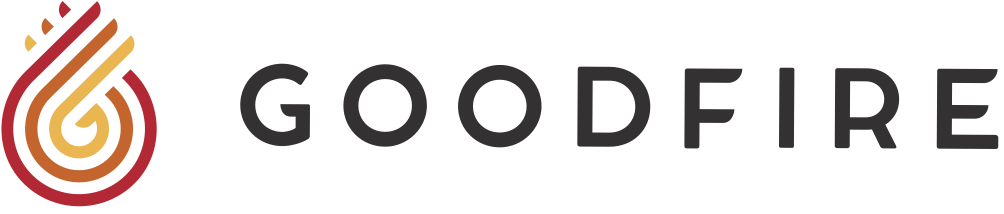} 
}

\newcommand{\authstrut}{\rule[-0.6ex]{0pt}{3.2ex}}
\newcommand{\authorentry}[2]{\authstrut\aut{#1}\afn{#2}}
\newcommand{\afflabel}[2]{\afn{#1}\af{#2}}
\newcommand{\authorsep}{\quad}
\newcommand{\repolink}[1]{%
    {\small
      \href{#1}{\raisebox{-2.8pt}{\includegraphics[height=10pt]{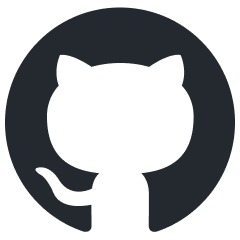}}%
    \textcolor{secondary}{~\nolinkurl{#1}}}
    }%
  }

\newtoggle{goodfireonly}
\togglefalse{goodfireonly}  %
\newcommand{\extline}[1]{\iftoggle{goodfireonly}{}{#1}}
\newcommand{\extaff}[1]{\iftoggle{goodfireonly}{}{#1}}

\newcommand{\paperauthors}{%
\authorentry{Thomas Fel}{\costar,\bullet} \authorsep
\authorentry{Matthew Kowal}{\costar,\bullet} \authorsep
\authorentry{Mozes Jacobs}{\costar,\bullet,\extaff{a}} \authorsep
\authorentry{Dron Hazra}{\costar,\bullet} \authorsep
\authorentry{Usha Bhalla}{\costar,\bullet} \\
\vspace{2pt}
\authorentry{Lee Sharkey}{\bullet} \authorsep
\authorentry{Lucius Bushnaq}{\bullet} \authorsep
\authorentry{Satchel Grant}{\bullet} \authorsep
\authorentry{Tal Haklay}{\bullet}
\\
\authorentry{Thomas Icard}{\bullet,\extaff{b}} \authorsep
\authorentry{Can Rager}{\bullet} \authorsep
\authorentry{Michael Pearce}{\bullet} \authorsep
\authorentry{Daniel Wurgaft}{\bullet,\extaff{b}} \authorsep
\\
\authorentry{Aiden Swann}{\bullet,\extaff{b}} 
\authorsep
\authorentry{Fenil Doshi}{\bullet,\extaff{a}} \authorsep
\authorentry{Siddharth Boppana}{\bullet} 
\authorsep
\authorentry{Curt Tigges}{\bullet}
\\
\vspace{2pt}
\authorentry{Nick Cammarata}{\bullet} \authorsep
\authorentry{Thomas Serre}{\extaff{c}} \authorsep
\authorentry{Vasudev Shyam}{\bullet} \authorsep
\authorentry{Owen Lewis}{\bullet}
\\
\vspace{2pt}
\authorentry{Thomas McGrath}{\bullet} \authorsep
\authorentry{Jack Merullo}{\colead,\bullet} \authorsep
\authorentry{Ekdeep Singh Lubana}{\colead,\bullet} \authorsep
\authorentry{Atticus Geiger}{\colead,\bullet}
\vspace{4pt}\\
$^\costar$\af{Equal contribution} \authorsep
\vspace{2mm}
$^\colead$\af{Equal senior contribution} \\
\goodfireaff \\
\vspace{3mm}
\repolink{https://github.com/goodfire-ai/block-sparse-featurizer}
\\
\vspace{1mm}
\extline{%
  \afflabel{\bullet}{Goodfire}\authorsep
  \afflabel{a}{Harvard University}\authorsep
  \afflabel{b}{Stanford University}\authorsep
  \afflabel{c}{Brown University}\authorsep
}
\vspace{-4mm}
}

\author{\paperauthors}

\title{Structuring Sparsity: Block-Sparse Featurizers Capture Visual Concept Manifolds}

\begin{document}
\vspace*{-22.5mm}

\maketitle

\vspace{-6mm}
\begin{figure}[H]
  \vspace{-3mm}
  \centering
  \hspace*{-0.0\linewidth}
  \includegraphics[width=0.99\linewidth]{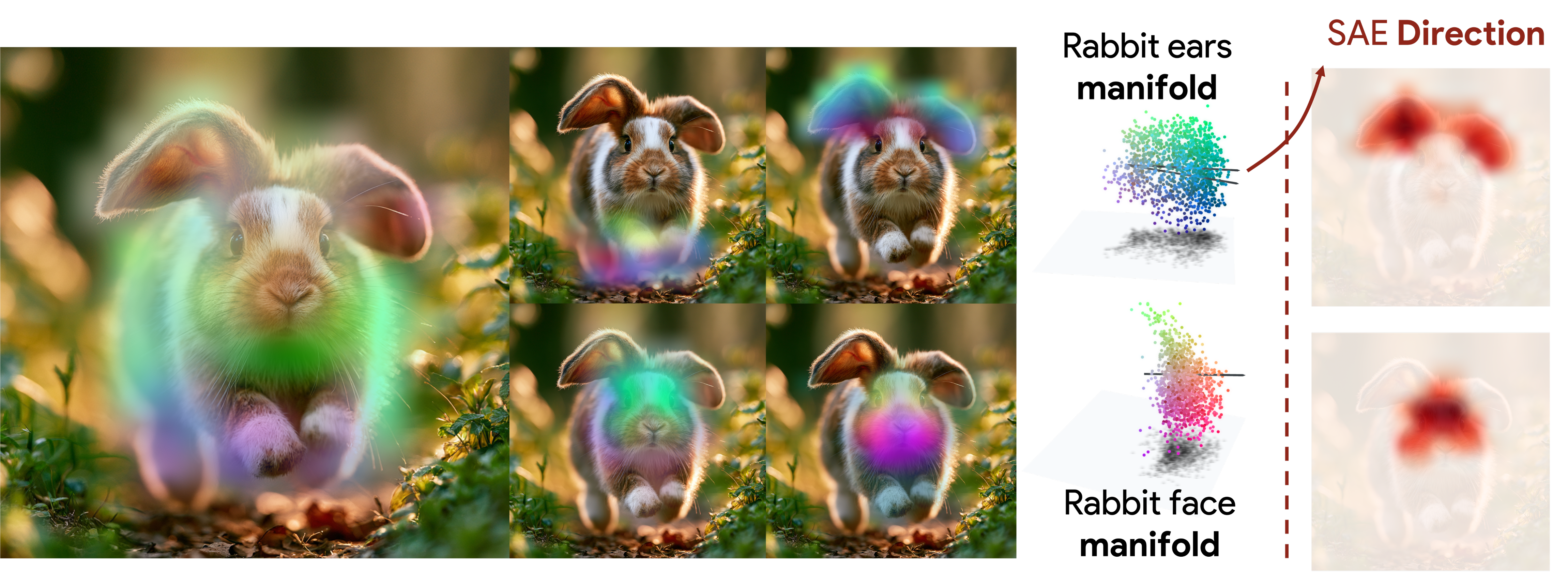}
  \vspace{-4.5mm}
\caption{\textbf{Block-sparse featurizers (BSF) capture the internal geometry of concepts.} Sparse autoencoders (SAEs) are a popular method for decomposing neural representations using isolated directions as primitive atoms. 
On the right, we show how two SAE features activate over the image of a rabbit where darker red means a higher activation; observe that one feature picks out the rabbit's ears and the other its face.
Yet, recent work points to more structure in these models than a single direction allows.
BSFs lift the primitive from directions to regions of activation space, which lets us recover the ear and face concepts as ``manifolds'' whose internal coordinates map out a conceptual space.
On the left, we show how blocks of directions activate over the image where different colors correspond to different locations within a block.
The exposed geometry is richer than an isolated direction would admit; for example, the face concept resolves into fine-grained facial features where the nose and the eyes are represented as different regions.
}
  \label{fig:intro}
  \vspace{-4.0mm}
\end{figure}

\enlargethispage{3\baselineskip}
\begin{abstract} 
What is the geometry of a visual percept? 
The most widely used protocols for decomposing neural network representations into interpretable parts treat concepts as isolated directions,
yet recent work shows that concepts are often realized as geometric structures in low dimensional regions of activation space. We turn to the literature of \textit{Structured sparsity} to close this gap, and show that block sparsity, which groups directions into blocks, is the prior matched to a generative model in which a representation is a sparse sum of low-dimensional manifolds---the modern, learned form of a classical idea in visual neuroscience, where a visual feature is carried by a coordinated group of neurons rather than a single tuned one. We implement three variants of block-sparse featurizers (BSFs) and, through a minimum-description-length analysis, show that all three describe activations more compactly than direction-based featurizers, with the recovered concepts typically two- to four-dimensional. 
We then use BSFs to \i{i} recontextualize prior work, showing that curve detectors in InceptionV1 actually read from a single continuous curve manifold, \i{ii} discover novel manifolds including shadows and lighting in DINOv3, and \i{iii} support interpretable control of image generation in diffusion models (SDXL) via manifold steering.
\end{abstract}

\clearpage

\section{Introduction} 

Shown nothing but natural images and told only to be sparse, the coding algorithm of~\cite{olshausen1996emergence} recovered the oriented, localized receptive fields of the primary visual cortex. 
Two consequences ensued at once: it lent substance to \citeauthor{barlow1961possible}'s hypothesis that the human brain is shaped to code perceptual data efficiently, and it in turn gave rise to sparse dictionary learning~\citep{tovsic2011dictionary,rubinstein2010dictionaries,elad2010sparse,mairal2014sparse,dumitrescu2018dictionary}.

Decades later, AI interpretability \citep{sharkey2025open} has adopted this program to decompose neural representations into sparse combinations of atomic units. Indeed, sparse autoencoders (SAEs;~\citealt{cunningham2023sparse, bricken2023monosemanticity, gao2024scaling,bussmann2024batchtopk,rajamanoharan2024jumping,klindt2023identifying,klindt2025superposition,costa2025flat,thasarathan2025universal})---the most widely used \textit{featurizer} for decomposing artificial neural activations---are an instantiation of sparse dictionary learning (cf.  \citealt{ghorbani2017interpretation,hindupur2025projecting,fel2023craft,fel2023holistic,zhang2021invertible,vielhaben2023multi,parekh2024concept,kowal2024visual,kowal2024understanding}).

However, SAEs assume that the proper atoms are isolated directions, while recent work demonstrates concepts can be represented as dense features~\citep{sun2025dense,fel2025into,lubana2025priors}, convex regions~\citep{fel2025into,tvetkova2025convex,park2024geometry,park2026information} and manifolds~\citep{chung2018classification,chung2021neural, modell2025,kantamneni2025language,engels2024not,gurnee2025when,yocum2025neural,karkada2026symmetry,feucht2026arithmetic,wurgaft2026manifold,bhalla2026sparse,sarfati2026shape}.\footnote{Neuroscience has moved along the same axis in two settings at once. In systems neuroscience, population geometry and neural manifolds~\citep{churchland2012neural,gallego2017neural,chung2018classification,chung2021neural} together with the mixed selectivity of single neurons~\citep{rigotti2013importance,ebitz2021population} carried the unit of representation away from the lone tuned cell toward something distributed and multi-dimensional; in the modeling of vision, independent subspace analysis learned from natural-image statistics features that are subspaces rather than directions and recovered the phase invariance of complex cells~\citep{hyvarinen2000emergence,hyvarinen2001topographic} in the lineage of the energy model~\citep{adelson1985spatiotemporal,karklin2009emergence}.}
This mismatch between the elicited phenomenology and the current apparatus is our starting point. 

To close the gap, we ought to return to what a featurizer is: a hypothesis about the data-generating process of the activations~\citep{hindupur2025projecting}. 
Read this way, the mismatch is resolved by matching the featurizer to the process the evidence supports,
and the emerging picture is that these structures share a property SAEs currently ignore: a concept carries an internal geometry.
The atoms of the featurizer must then be lifted accordingly, from individual directions to blocks of them. This idea is well established under the name of block sparsity---we also use the term group sparsity---a central instance of the broader \textit{Structured sparsity} literature\footnote{We use \emph{Structured sparsity} in its classical sense: sparsity with a prior over admissible supports. Rather than selecting isolated coordinates, structured sparse models select groups, blocks, trees, overlapping families, or other organized configurations. Specific instances include the group lasso \citep{yuan2006model}, block-sparse recovery \citep{eldar2009block}, and hierarchical or overlapping-group sparse coding \citep{jenatton2010structured,mairal2014sparse,bach2012structured}. A contribution of this paper is to explicitly highlight the fundamental interpretability questions that the structured sparsity literature deeply engages, e.g., \textit{What should be unit of analysis for decomposing a representation?}, \textit{What is the admissible geometry of such a unit?}, or \textit{How can the sparsity prior be adapted to different types of representations?}}~\citep{yuan2006model,eldar2009block,jenatton2010structured,bach2012structured}.

Work in language model interpretability has begun to explore the same relaxation, proposing featurizers that move beyond the one-direction-per-concept assumption~\citep{francel2026smixae,dalili2026subspace,hindupur2025projecting,shafran2026directions}.
The crucial point for our purposes is that these methods instantiate, in modern interpretability language, a much older organizing idea: sparsity need not be placed on individual coordinates, but can instead be imposed over structured supports.
This paper makes that connection explicit and argues for the principle of block sparsity as an inductive bias for decomposing neural representations by:
\i{i} deriving the generative motivation, \i{ii} building practical featurizers around it, and \i{iii} analyzing the structures they recover in neural representations.

Our contributions are as follows:

\begin{itemize}[nosep,leftmargin=*]

\item \textbf{Representation geometry through the lens of structured sparsity.} 
We connect the claims that concepts are represented by subspaces and manifolds in neural representations to the structured sparsity literature. We show that block sparsity is the inductive bias matched to a generative process in which each representation is a sparse sum of low-dimensional manifolds.

\item \textbf{Three block-sparse featurizers (BSFs).} 
We introduce three BSFs that instantiate this principle: (1) the \textit{Vanilla BSF} learns groups of directions where only the top $k$ blocks with the largest norms are used in reconstruction; (2) the \textit{Grassmannian BSF} where the encoder is the transpose of the decoder and enforces that directions within a block are orthogonal; and (3) the \textit{Group Lasso featurizer} with linear encoder soft-threshold activation function and with a group lasso penalty.

\item \textbf{Quantifying concept dimensionality.} We train our three BSF variants on DINOv3 and perform a minimum description length analysis to compare BSF variants with each other and SAEs. We find that grouping directions into blocks yields more efficient descriptions of activations. A stable-rank study lets us quantify the intrinsic dimensionality of individual features, revealing a spectrum of dimensions of features.%

\item \textbf{Discovering concept manifolds in vision models.} We use BSFs to discover previously unknown concept manifolds in vision models.  
In InceptionV1, we show that the famous curve-detector neurons \citep{cammarata2020curve, gorton2024missing} all read off of a previously undiscovered curve manifold and expose additional Fourier modes. %
In DINOv3, we uncover manifolds that correspond to concepts that abstract from the details of objects in a scene, e.g., their lighting and shadows. 
In SDXL, we identify a variety of  manifolds along which generation can be steered.
\end{itemize}

\section{Structured Sparse Model of Feature Geometry}

\begin{wrapfigure}{r}{0.38\textwidth}
    \centering
    \includegraphics[width=0.97\linewidth]{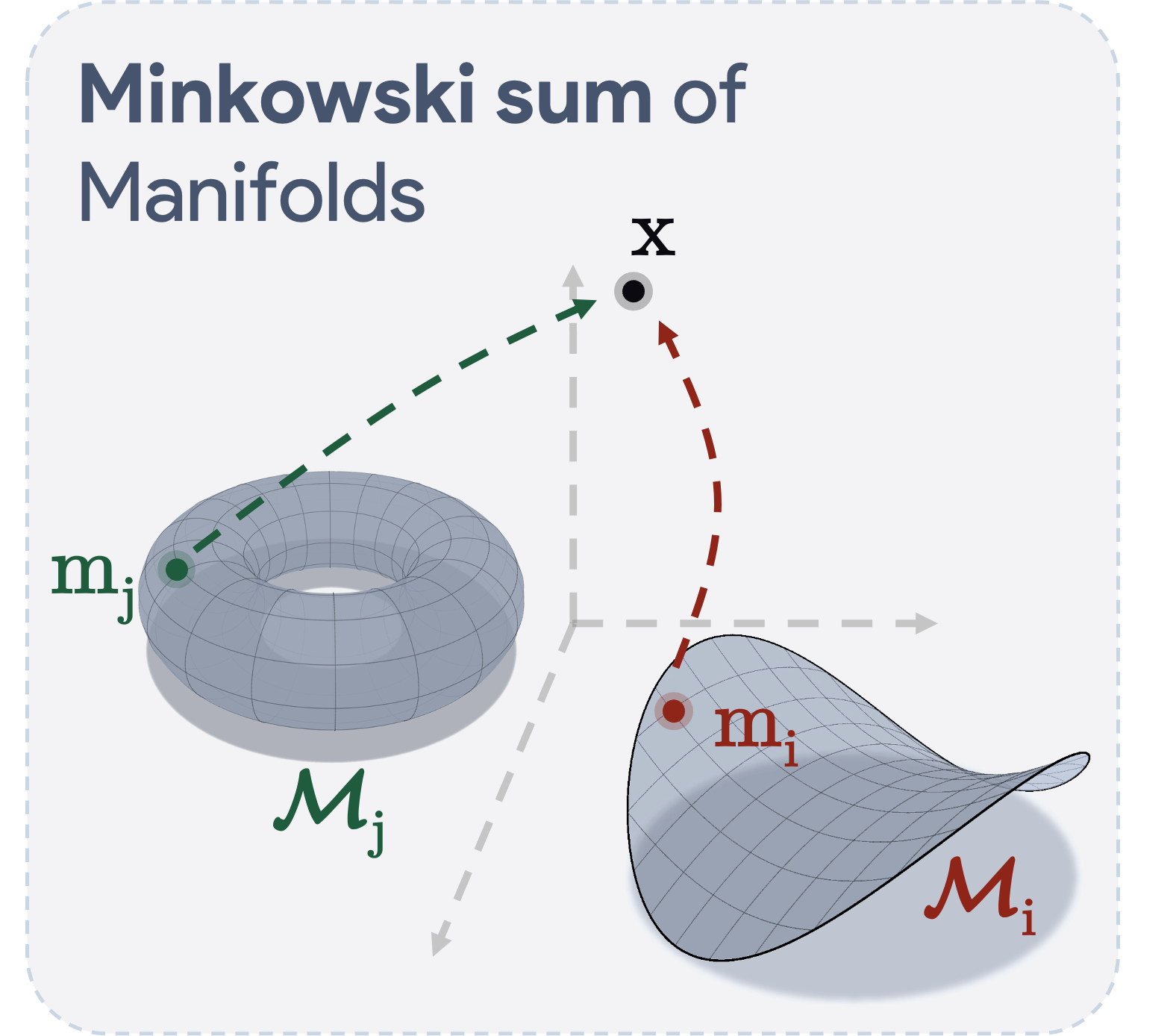}
    \caption{Activation $\bm{x}$ as a sparse sum of points drawn from a few low-dimensional regions $\mathcal{M}_i$, $\bm{x}$ lives in a Minkowski sum of manifolds (Def.~\ref{def:amm}).}
    \label{fig:dgp}
\end{wrapfigure}

There is a deep duality between the architecture of a featurizer and the assumed geometry of neural representations \citep{hindupur2025projecting}. To design a new featurizer, we first hypothesize a data-generating process (DGP) based on recent work on representation manifolds~\citep{fel2025into, lubana2025priors, bhalla2026sparse, modell2025, gurnee2025when}. We will then show that the natural featurizer architecture for this DGP is a block-sparse method. 

\paragraph{General notation.}
We work in the standard representation-learning setting. For $n \in \mathbb{N}$, we write $[n] = \{1,\dots,n\}$. A neural network model maps an input to an activation vector $\bm{x} \in \mathbb{R}^d$, and we study the set of activations $\mathcal{A} \subset \mathbb{R}^d$ it induces over a data distribution.
For a code (e.g., SAEs coefficients) $\bm{z} \in \mathbb{R}^{\blocksize G}$ partitioned into $G$ blocks $\bm{z} = (\bm{z}_1,\dots,\bm{z}_G)$ with $\bm{z}_g \in \mathbb{R}^{\blocksize}$, we write $\operatorname{supp}(\bm{z}) = \{i : z_i \neq 0\}$ and $\|\bm{z}\|_0 = |\operatorname{supp}(\bm{z})|$ for usual support and $\operatorname{supp}_{\mathcal{G}}(\bm{z}) = \{g : \bm{z}_g \neq \bm{0}\}$ for the block support (the number of active blocks).

\subsection{Data-Generating Process} 
Following~\citet{bhalla2026sparse}, we adopt the following model of representations, in which an activation is composed from a sparse set of manifolds immersed in low-dimensional regions.

\begin{definition}[Additive Mixture of Manifolds~\citep{bhalla2026sparse}]
\label{def:amm}
Let $\mathcal{N}_1, \dots, \mathcal{N}_M$ be abstract manifolds with $\dim \mathcal{N}_i \ll d$, immersed in activation space through $\bm{\gamma}_i : \mathcal{N}_i \to \mathbb{R}^d$ with images $\mathcal{M}_i = \bm{\gamma}_i(\mathcal{N}_i) \subset \mathbb{R}^d$. An activation $\bm{x} \in \mathcal{A}$ follows the additive mixture of manifolds model if
\begin{equation}
\label{eq:dgp}
\bm{x} = \sum_{i \in S} \bm{m}_i, \qquad \bm{m}_i \in \mathcal{M}_i, \qquad S \subseteq [M], \quad |S| \ll M,
\end{equation}
where $S$ indexes the factors active in $\bm{x}$. Equivalently, $\bm{x}$ lies in a Minkowski sum of $M$ manifolds.
\end{definition}

Importantly, we treat Def.~\ref{def:amm} as a hypothesized model of representations rather than an assertion that neural activations exactly satisfy it; its role is to specify the unit of composition.
Under this manifold model, each active factor contributes a point on a low-dimensional manifold; recovery asks two questions at once: \i{i} which factors are present, and \i{ii} where within each factor the activation lies.
These two aspects should not be regularized in the same way.
The first is a sparse selection problem over manifolds, while the second is a manifold learning problem.
In the next section, we will show that block sparsity is \textit{precisely} the prior that separates these roles: it enforces sparsity across factors while allowing dense variation inside each active factor.

\subsection{Block Sparsity as the Matched Prior for Manifold Superposition}

One may already sense the connection between this manifold-based data generating process and block sparsity. Here we make it explicit: under a Bayesian treatment of the model, the group-sparse penalty $\|\bm{z}\|_{2,0}$ falls out of the maximum a posteriori estimate, arising as the cost of switching a factor on (full argument in Appendix~\ref{app:matched_prior}).
We add one geometric assumption, i.e., that each factor occupies a low-dimensional linear subspace of activation space, $\mathcal{V}_g = \operatorname{span}(\mathcal{M}_g)$ with $\dim \mathcal{V}_g = \blocksize \ll d$.
This condition is strictly stronger than low intrinsic dimension since a factor of intrinsic dimension $r$ may still spread across many ambient directions.\footnote{We expect this of trained representations: keeping co-active factors in direct sum requires $|S|\,\blocksize \leq d$, so large spans force their subspaces to overlap, and the shared directions are the interference a representation is pressured to minimize. This is the subspace analogue of the RIP condition (block-RIP,~\citealp{eldar2009block}), under which the active subspaces stay incoherent.}

Resolved in the frame of each subspace, the superposition is a sum of block contributions. Each active manifold contributes a point $\bm{m}_g \in \mathcal{V}_g$, and fixing an orthonormal frame $\bm{D}_g \in \mathrm{St}(\blocksize, d)$ for $\mathcal{V}_g$, a point on the Stiefel manifold of orthonormal $\blocksize$-frames in $\mathbb{R}^d$, writes that point in coordinates $\bm{z}_g \in \mathbb{R}^{\blocksize}$ as $\bm{m}_g = \bm{z}_g \bm{D}_g$, so that collecting the coordinates into $\bm{z} = (\bm{z}_1,\dots,\bm{z}_G)$ and writing $\bm{\varepsilon} \sim \mathcal{N}(\bm{0}, \sigma^2 \bm{I}_d)$ for the observation noise,
\begin{equation}
\label{eq:linearized_dgp}
\bm{x} = \sum_{g \in S} \bm{m}_g + \bm{\varepsilon} = \sum_{g \in S} \bm{z}_g \bm{D}_g + \bm{\varepsilon}, \qquad \operatorname{supp}_{\mathcal{G}}(\bm{z}) = S.
\end{equation}
If we place a spike-and-slab prior on each block~\citep{mitchell1988bayesian,soussen2011bernoulli}
the recovery problem follows directly.
\begin{proposition}[Block sparsity is the maximum a posteriori (MAP)-matched prior]
\label{lemma:matched_prior}
Let $\rho = \mathcal{U}(\mathcal{B}_R)$ be the uniform density on the ball $\mathcal{B}_R = \{\bm{u} \in \mathbb{R}^{\blocksize} : \|\bm{u}\|_2 \leq R\}$, and $\delta_{\bm{0}}$ the point mass at $\bm{0}$. Under the model of Eq.~\ref{eq:linearized_dgp} with the dictionary $\bm{D}$ fixed and independent spike-and-slab block priors
\begin{equation*}
p(\bm{z}_g) = (1-\pi)\,\delta_{\bm{0}} + \pi\,\rho, \qquad \pi \in (0,1),
\end{equation*}
the MAP estimate of the code is
\begin{equation}
\label{eq:block_l0}
\hat{\bm{z}} = \argmin_{\bm{z}} \tfrac{1}{2}\|\bm{x} - \bm{z}\bm{D}\|_2^2 ~+~ \lambda \|\bm{z}\|_{2,0},
\end{equation}
with $\lambda = \sigma^2 \log\big(\tfrac{1-\pi}{\pi}\operatorname{vol}(\mathcal{B}_R)\big)$, where $\bm{D} = (\bm{D}_1,\dots,\bm{D}_G) \in \mathbb{R}^{\blocksize G \times d}$ stacks the frames and $\bm{z} = (\bm{z}_1,\dots,\bm{z}_G)$ their coordinates. Proof in Appendix~\ref{app:matched_prior}.
\end{proposition}
The term of interest that explicitly appears here is  $\|\bm{z}\|_{2,0}$, which represents the number of distinct, non-overlapping groups that activate---i.e., block sparsity.

\begin{wrapfigure}{r}{0.37\textwidth}
    \vspace{-2mm}
    \centering
    \includegraphics[width=0.97\linewidth]{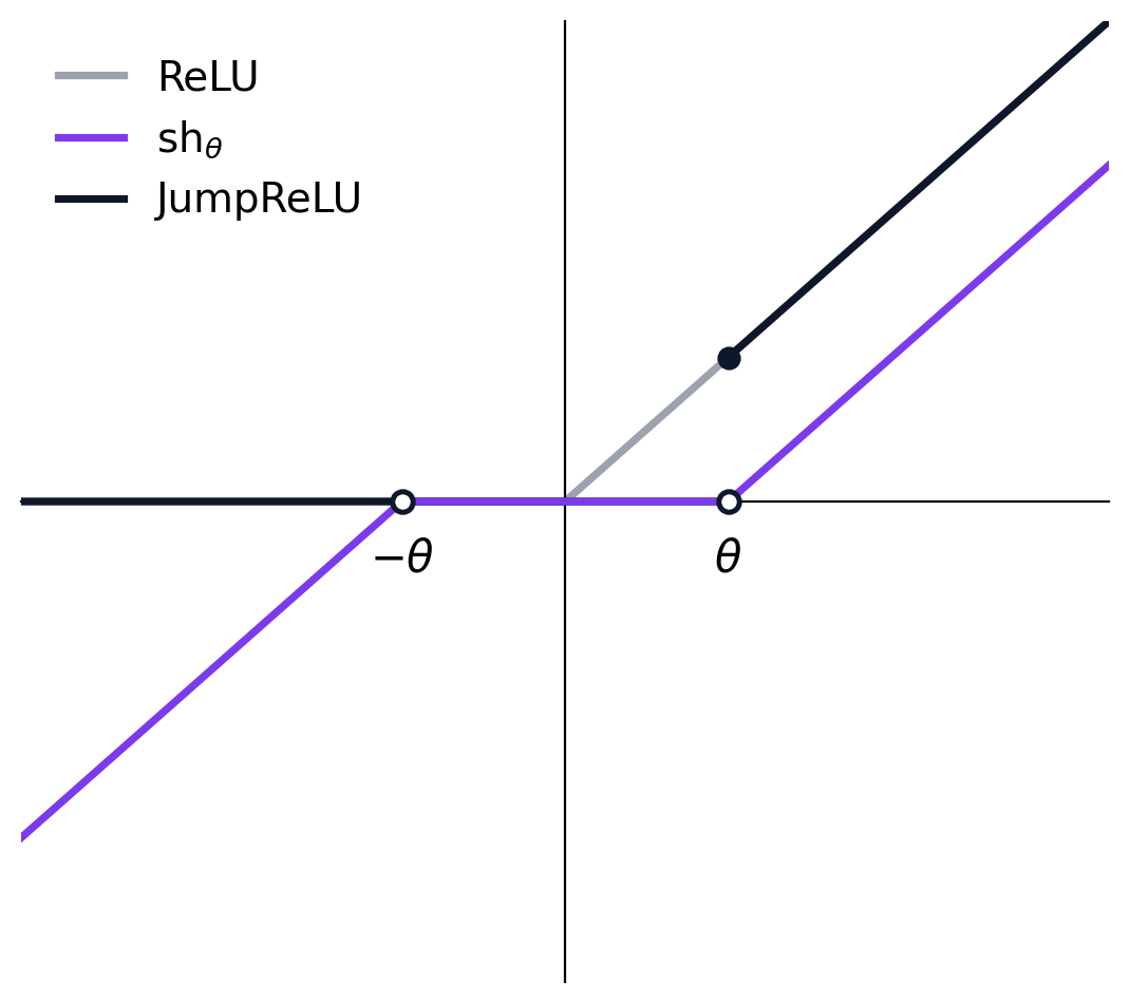}
    \vspace{-1mm}
    \caption{\textbf{The block soft-threshold.} The shrinkage operator (applied on the norm of the group) against ReLU and JumpReLU~\citep{rajamanoharan2024jumping}.}
\end{wrapfigure}
A few remarks follow. 
First, the penalty sees each block only through its norm $\|\bm{z}_g\|_2$ because the same point $\bm{m}_g = \bm{z}_g \bm{D}_g$ can be written in any rotated frame, $\bm{z}_g \bm{D}_g = (\bm{z}_g \bm{Q}^\top)(\bm{Q} \bm{D}_g)$ for $\bm{Q} \in \mathrm{O}(\blocksize)$, 
so a penalty intrinsic to the feature must be ``blind'' to the choice of basis, which leaves only the $\ell_2$ norm. The atomic unit is thus the subspace itself, and not any particular basis of it.
Second, an astute reader will notice that this is just the block analogue of a classical correspondence between Bernoulli priors and $\ell_0$ penalties~\citep{soussen2011bernoulli,olshausen1997sparse}, lifted here to the level of factors~\citep{zhang2013extension,baraniuk2010model}.
Third, Eq.~\ref{eq:block_l0} settles what to optimize without settling how: the support selection is NP-hard in general~\citep{eldar2009block}, admitting no closed form solution, so any practical featurizer must commit to an approximation strategy, whether by relaxing the penalty, selecting blocks greedily, or projecting onto the constraint, and we propose three instantiations of the principle in the next subsection.
\subsection{Block-Sparse Featurizers}
\label{subsec:bsf}

We now propose three Block-Sparse Featurizers (BSFs), and stress that they are candidate architectures: each is one way of implementing the same principle (block sparsity). Our experiments support the general principle transfers, while revealing how concrete implementations trade off against each other.
All three BSFs share the linear decoder $\hat{\bm{x}} = \bm{z}\bm{D}$ of Eq.~\ref{eq:linearized_dgp} and differ only in how the code $\bm{z}$ is produced.

To write them compactly, we introduce two operators that act block-wise, involving each block only through its $\ell_2$ norm. The block projection $\bm{\Pi}_{\topk}$ keeps the $\topk$ blocks of largest norm $\|\bm{z}_g\|_2$ and zeroes the rest---the block analogue of absolute TopK~\citep{gao2024scaling,zhu2025abstopk} and a projection onto the block-sparsity constraint $\|\bm{z}\|_{2,0} \leq \topk$. The second object we need to define is the block soft-threshold activation function~\citep{yuan2006model,puig2009multidimensional}, $\text{sh}_{\theta}$, that shrinks every block toward zero,
\begin{equation*}
\text{sh}_{\theta}(\bm{z})_g = \max(1 - \tfrac{\theta}{\|\bm{z}_g\|_2}, 0) \bm{z}_g,
\end{equation*}
where $\theta \in \mathbb{R}$. This is exactly the proximal operator of the $\ell_{2,1}$ norm~\citep{bach2012structured}. Because both block-wise operators depend on a block only through $\|\bm{z}_g\|_2$, the model selects subspaces and not a particular basis within them. 
This norm-based selection is also why no coordinate-wise nonlinearity appears: a \texttt{ReLU} or \texttt{JumpReLU} would keep only non-negative codes and so restrict each block to a cone rather than the full subspace, a restriction not motivated by Def.~\ref{def:amm} and one that, as we show later, fails to recover superpositions a signed code resolves.
The three BSF architectures are then:
\begin{equation}
\label{eq:bsf}
\bm{z}(\bm{x}) =
\begin{cases}
\;\bm{\Pi}_{\topk}\big(\bm{x}\bm{W} + \bm{b}\big) & \text{(Vanilla BSF)} \\[3pt]
\;\bm{\Pi}_{\topk}\big(\gamma \, \bm{x}\bm{D}^\top\big), \quad \bm{D}_g \in \mathrm{St}(\blocksize, d) & \text{(Grassmannian BSF)} \\[3pt]
\;\text{sh}_{\bm{\theta}}\big(\bm{x}\bm{W} + \bm{b}\big) & \text{(Group Lasso BSF)}
\end{cases}
\end{equation}
All three BSF variants are trained to reconstruct $\min \|\bm{x} - \bm{z}\bm{D}\|_2^2$. The Group Lasso featurizer includes the additional penalty $\lambda \|\bm{z}\|_{2,1}$, while the other two enforce sparsity by construction.
They correspond to three approaches to the problem of Eq.~\ref{eq:block_l0} having no closed form solution. 
The vanilla BSF keeps the constraint on the code and projects, learning a free encoder $\bm{W}$ and decoder $\bm{D}$ untied from each other, so that a linear map produces the code and $\bm{\Pi}_{\topk}$ selects the $\topk$ active blocks. 
The Grassmannian BSF moves the constraint to the dictionary: each block is an orthonormal chart, encoder and decoder are tied, and a single learned scalar $\gamma$ compensates the energy lost by tying. 
The Group Lasso BSF shares this free linear encoder but replaces the projection with the soft-threshold $\text{sh}_{\theta}$, so that blocks are selected by shrinkage rather than by a hard count, relaxing the penalty to its convex surrogate~\citep{yuan2006model} and trading exact sparsity for a smooth objective.
We defer practical details of training to Appendix~\ref{app:implementations}.

The BSFs are ready to train, and the goal of the next section is to evaluates them. We first put the matched-prior claim of Def.~\ref{def:amm} to a controlled test, on a synthetic superposition of manifolds whose ground-truth factors are known by design.
We then turn to real activations, where two choices the objective leaves open must be settled: at what fidelity an activation counts as reconstructed, and how large the block dimension $\blocksize$ should be, since reconstruction alone favors making it large (Fig.~\ref{fig:pareto}, bottom). An information-theoretic reading then locates the sweet spot for both.

\section{Evaluating Block Sparse Featurizers}

\begin{figure}[t]
    \vspace{-10mm}
    \centering
    \includegraphics[width=0.99\linewidth]{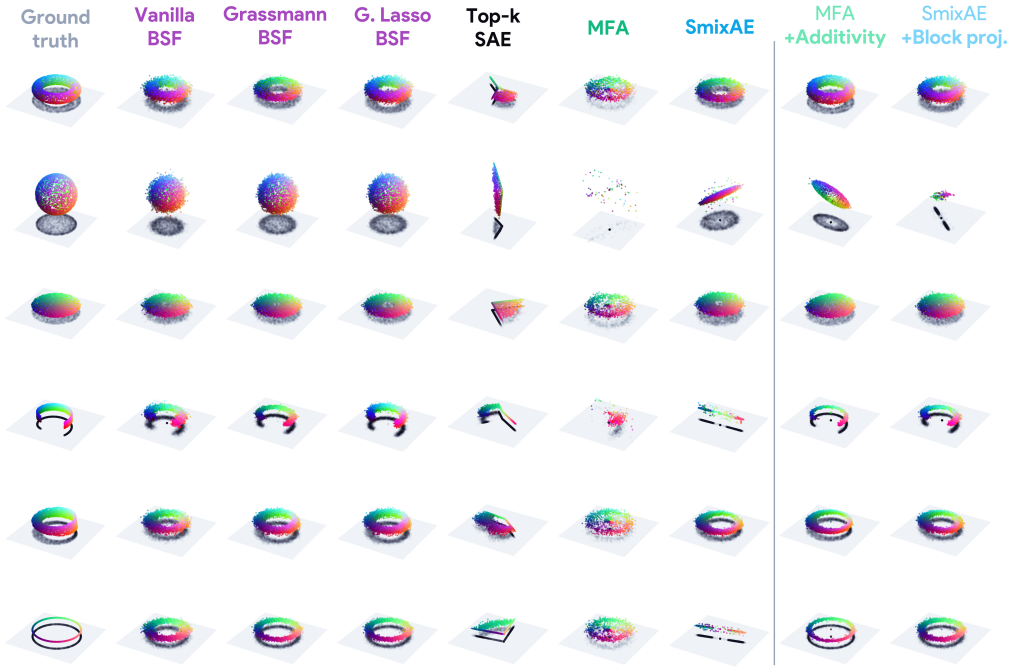}
    \caption{\textbf{Block sparsity recovers an additive manifold superposition.}
    A controlled instance of Def.~\ref{def:amm}: $M$ known low-dimensional manifolds embedded in $\mathbb{R}^d$ and summed $|S|$ at a time (here six factors, one per row). The leftmost column is the ground-truth contribution $\bm{m}_i$; each remaining column is the contribution recovered by a featurizer, projected into the true factor's $3$-D principal frame and colored by that frame (hue $=$ position on the manifold). The three BSFs (Grassmannian, Vanilla, Group~Lasso) return the manifolds faithfully; the classical SAE and the as-published SMIXAE and MFA featurizers shatter or collapse them; the two rightmost ``fixed'' columns apply the block-sparse repairs described in the text.}
    \label{fig:toy_recovery}
    \vspace{-7mm}
\end{figure}

\subsection{Toy model of Manifold Superposition}
\label{subsec:toma}

The data-generating process of neural network activations has underlying factors that are unobserved, and so recovery ``in the wild'' can be assessed only indirectly.
Therefore, we begin in a controlled setting where the data-generating process is known by design and so we can test the central claim of Def.~\ref{def:amm} head on: that block sparsity is the prior matched to an additive superposition of low-dimensional manifolds. 
Following Def.~\ref{def:amm}, we generate synthetic data as in Eq.~\ref{eq:dgp}: $M$ primitive factors, a mix of one-dimensional concept atoms and curved manifolds (circles, spheres, tori, ...), are embedded into $\mathbb{R}^d$ through random orthonormal maps, and each data point is the sum of $|S| \ll M$ of them. Because the active set $S$ and every contribution $\bm{m}_i$ are known, recovery can be scored directly: we match each primitive to the block whose firing best predicts its presence, and report, per primitive, the fraction of the variance of $\bm{m}_i$ that block reconstructs ($R^2$). Full settings are deferred to Appendix~\ref{app:toma}.

\textbf{BSFs recover the superposition.} Figure~\ref{fig:toy_recovery} shows the recovered manifolds, rendered in the two-resolution reading of Sec.~\ref{subsec:bsf} (block norm for presence, PCA-of-contributions for intrinsic geometry), and Figure~\ref{fig:toy_leaderboard} the per-block recovery. All three BSFs recover the factors close to the oracle ceiling---per-block $R^2$ between $0.93$ and $0.97$, against an oracle of $0.99$---using a single amortized forward pass with no iterative inference. The classical TopK SAE instead shatters each multi-dimensional factor across atoms and cannot isolate a single primitive's contribution, so its per-block reconstruction collapses ($R^2 \approx 0.53$) and its recovered clouds degenerate to scattered slivers rather than the manifolds the BSFs return.
\begin{wrapfigure}{r}{0.45\textwidth}
    \vspace{-3mm}
    \centering
    \includegraphics[width=0.97\linewidth]{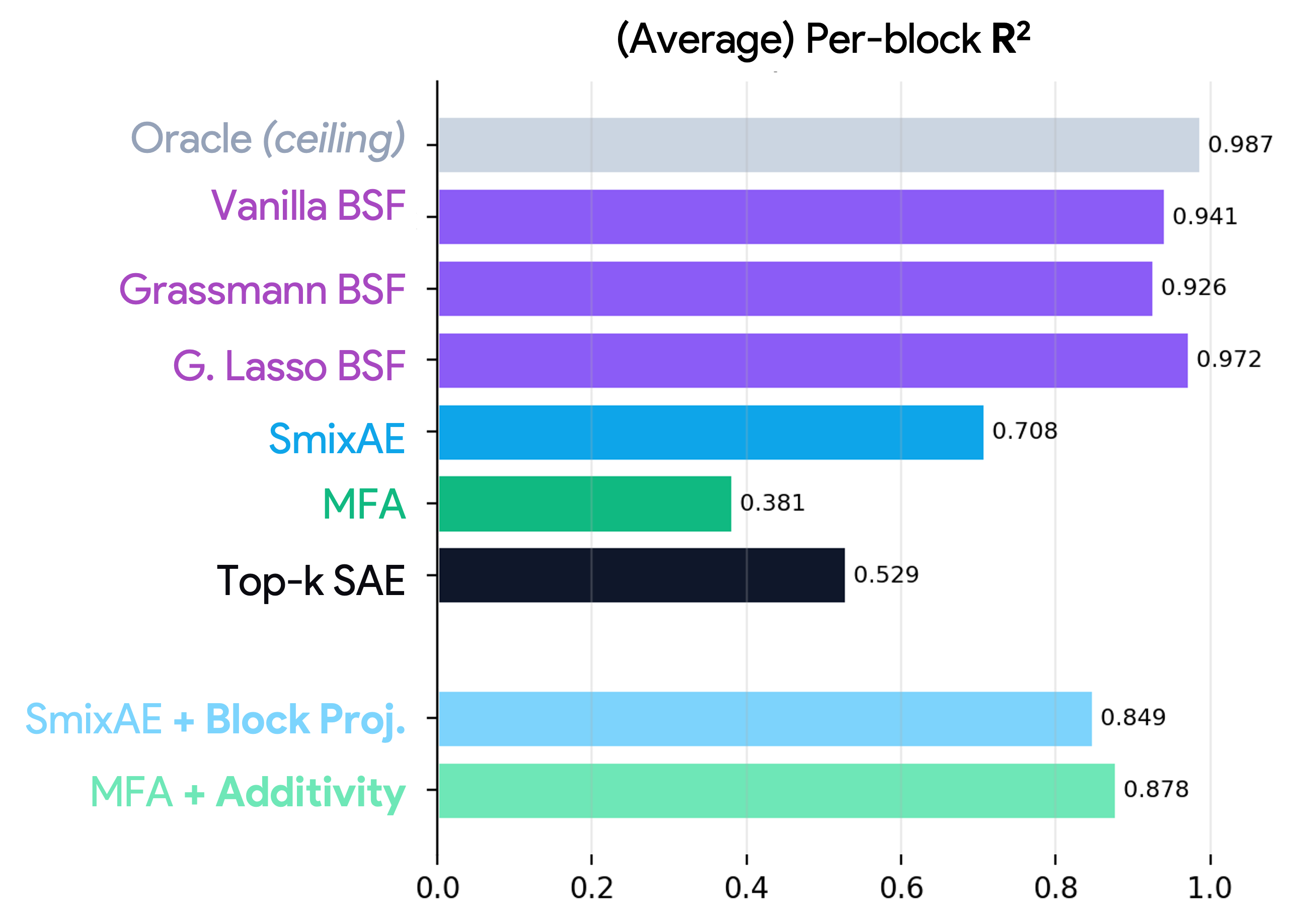}
    \vspace{-2mm}
    \caption{\textbf{Per-block recovery.} Per-block $R^2$, averaged over factors. The block-sparse featurizers (purple) sit near the oracle ceiling; the two repaired variants (bottom) recover well above their original forms.}
    \label{fig:toy_leaderboard}
    \vspace{-7mm}
\end{wrapfigure}
\textbf{The same principle repairs SMixAE and MFA.} The toy also lets us probe two recent featurizers that, like ours, take a multi-dimensional unit but are built on different assumptions: SMixAE~\citep{francel2026smixae} and MFA~\citep{shafran2026directions}. Both underperform on the additive toy (per-block $R^2 \approx 0.71$ and $0.38$), and in each case the cause is a departure from block sparsity that a small change can repair. 

For SMixAE, the authors already note its failure on this toy task, and we find that replacing the rectifier with a signed code (motivated by the cone argument of Sec.~\ref{subsec:bsf}) and substituting our block projection $\bm{\Pi}_{\topk}$ for its activation function lifts recovery to $R^2 \approx 0.85$.

MFA fails for a more fundamental reason, one that separates a mixture from a sum. 
The assumed data-generating process of MFA is a single component drawn per data point.
This means that the points it can generate lie in the union of its component subspaces rather than in their sum, and at inference it returns a convex combination of those components (weights sum to one).
When a factor and another are both present, the reconstructed data point lands in the subspace they jointly span, which no convex combination of the individual subspaces reaches, since such a combination stays within the hull of its constituents while a genuine sum extends beyond it.

Another way to see the difference is that the weights the model assigns express uncertainty over which one component is responsible rather than the co-presence of several, so MFA can recover the individual subspaces but not their sum (for reasons explained by~\cite{bhalla2026sparse}, Appendix C), which is why its per-block reconstruction caps low. 
Yet, decoding the learned subspaces additively rather than through this convex posterior raises recovery to $R^2 \approx 0.88$.
We stress that this additive decode is a heuristic relaxation rather than a coherent generative model, we therefore report it as a promising variant rather than a contribution of this work, and detail both fixes and our derivation of the implicit data-generating process of MFA in Appendix~\ref{app:toma}.

With recovery established against known ground truth, we turn from synthetic data to activations from neural networks, where no ground truth is available. Now, the featurizers can instead be judged by how compactly and faithfully they describe what the model computes.

\subsection{Faithfulness of Reconstruction}

\begin{figure}
    \vspace{-12mm}
    \centering
    \includegraphics[width=0.90\linewidth]{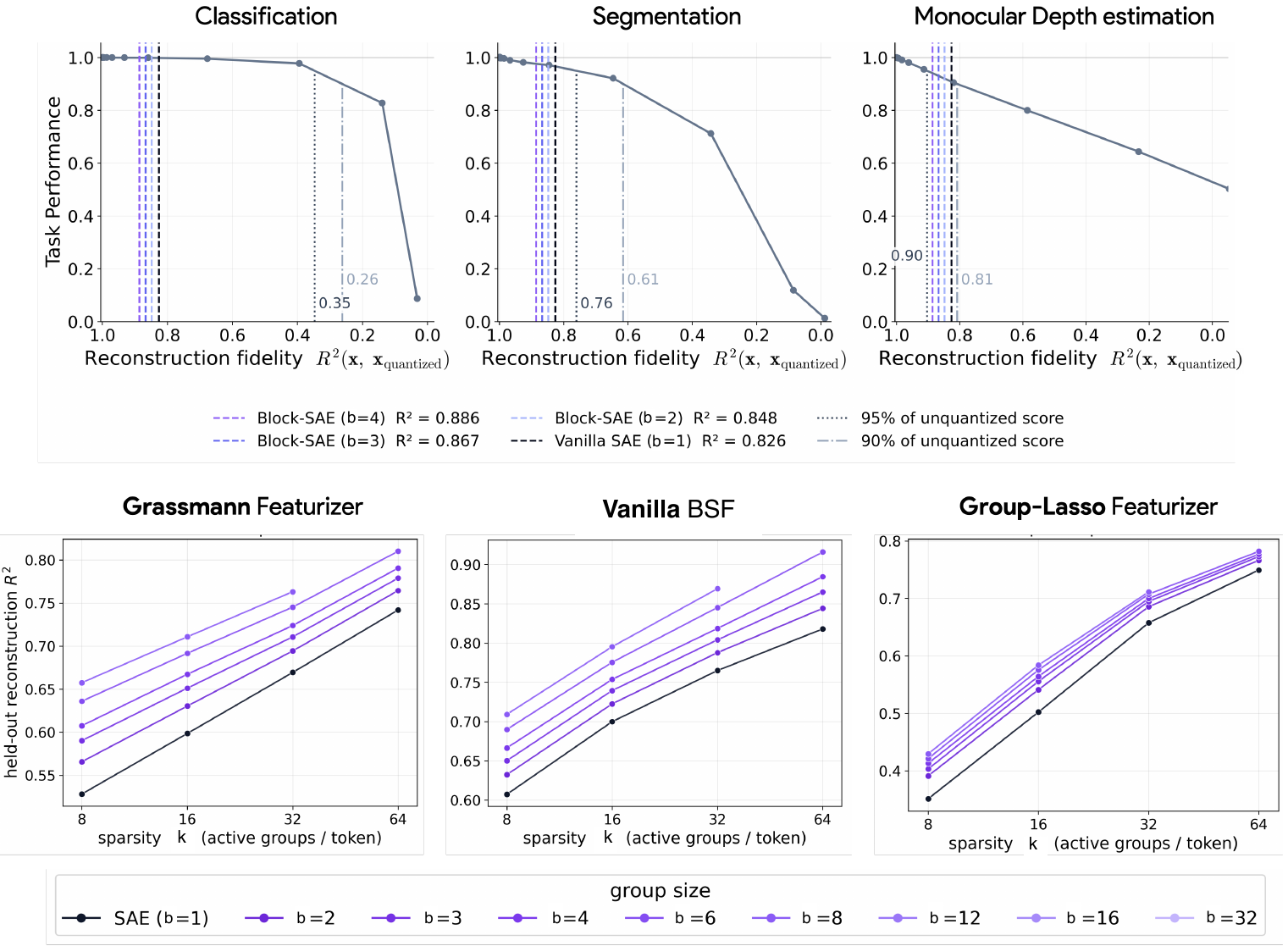}
    \caption{\textbf{Distortion estimation and reconstruction sparsity trade-off.} \textbf{(top)} Estimating the reconstruction fidelity each task requires. DINOv3 activations are degraded by progressive quantization and passed to a frozen linear probe; the relative task metric is plotted against the fidelity $R^2(\bm{x}, \bm{x}_q)$ of the corrupted features. Performance is essentially unaffected down to $R^2 \approx 0.8$ for classification and segmentation, whereas depth estimation is more demanding, requiring $R^2 \approx 0.9$ to retain $95\%$ of its accuracy. These task-dependent floors fix the distortion $\delta$ at which description length is read. \textbf{Held-out $R^2$ against block sparsity $\topk$ for the three featurizers. (bottom)} Lines are colored by block dimension $\blocksize$ ($\blocksize{=}1$ is the SAE where atoms are directions). Reconstruction rises monotonically with the dictionary width $G$, the sparsity $\topk$, and the block dimension $\blocksize$, so reconstruction fidelity alone cannot rank the featurizers and an external criterion is needed.}
    \label{fig:pareto}
\end{figure}
We begin by studying the degree to which block featurizers faithfully reconstruct activations by measuring variance explained ($R^2$). But in order to ground this, we first establish the noise floor at which a quantized representation can perform some task. This allows us to set a bar for reconstruction our featurizer must hit to be usefully faithful. We measure this by gradually quantizing representations  from DINOv3 more and more and measuring when performance degrades.

Figure~\ref{fig:pareto} (top) estimates this fidelity empirically. Degrading activations by progressive quantization and recording the effect on three downstream tasks, we find that the tolerance is task-dependent: classification and segmentation are essentially unaffected down to $R^2 \approx 0.8$, while depth estimation is more demanding, requiring $R^2 \approx 0.9$ to retain $95\%$ of its accuracy.
We take these task floors as the fidelity a featurizer must reach, reading each as the distortion its task tolerates without becoming less useful.

With the fidelity floor established, we turn to the Pareto frontiers of our featurizers in Figure~\ref{fig:pareto} (bottom), where reconstruction improves monotonically with every structural parameter, as the dictionary widens, as more blocks are permitted to fire, and as the block dimension grows. 
This monotonicity is expected, and it is also what prevents reconstruction quality from ranking the featurizers against one another. 
\subsection{Minimum Description Length as a Comparison Principle}
\label{subsec:mdl}
Next we compare the quality of reconstructions to Sparse Autoencoders (SAEs), which represent concepts one-dimensionally. The comparison is confounded at a fixed reconstruction budget, since the SAEs code is the less constrained of the two, free to activate any combination of atoms, whereas a block-sparse code is bound to its blockwise structure. A criterion that compares them on equal terms must therefore come from elsewhere, and the minimum description length principle (MDL) supplies one naturally, since prior work has appealed to it for similar reasoning~\citep{ayonrinde2024interpretability}.

\begin{figure}[t]
    \vspace{-12mm}
    \centering
    \includegraphics[width=0.99\linewidth]{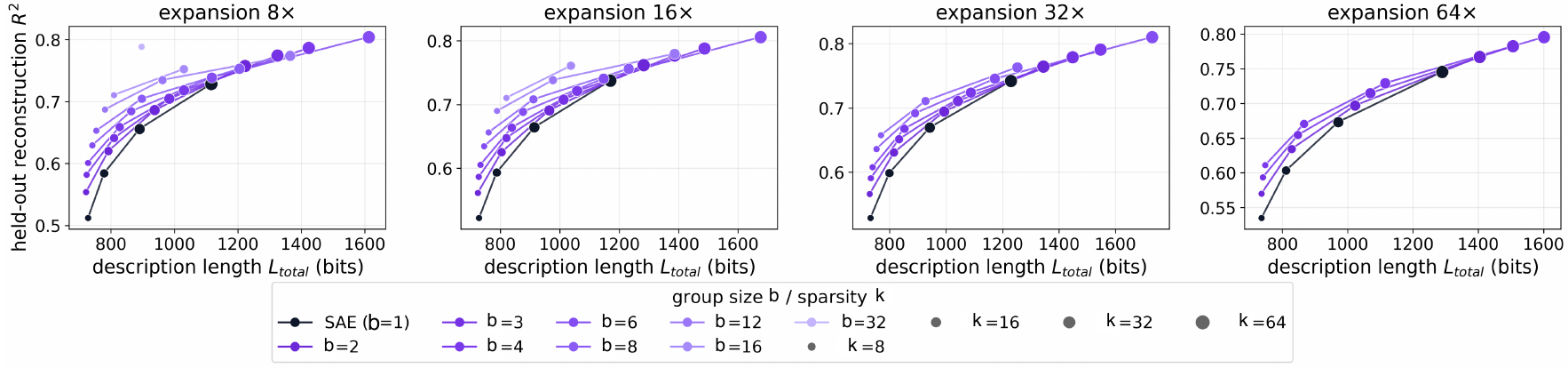}
    \caption{\textbf{Block structure yields a lower description length than SAEs.} Description length $L_\delta(\bm{x})$ (Eq.~\ref{eq:mdl}) for the Grassmannian BSF at the $20\%$ distortion floor, against block dimension $\blocksize$, with one panel per dictionary width $G$ and one curve per block sparsity $\topk$. Codes with $\blocksize{>}1$ describe DINOv3 activations in fewer bits than the $\blocksize{=}1$ SAE across every width and sparsity, the minimum falling at a moderate $\blocksize$ between $2$ and $4$ that eases downward as $G$ widens. Similar results hold for all three featurizers and the $10\%$ floor and are reported in Appendix~\ref{app:mdl}.}
    \label{fig:mdl_grassmann}
\end{figure}
\begin{figure}
    \centering
    \includegraphics[width=0.93\linewidth]{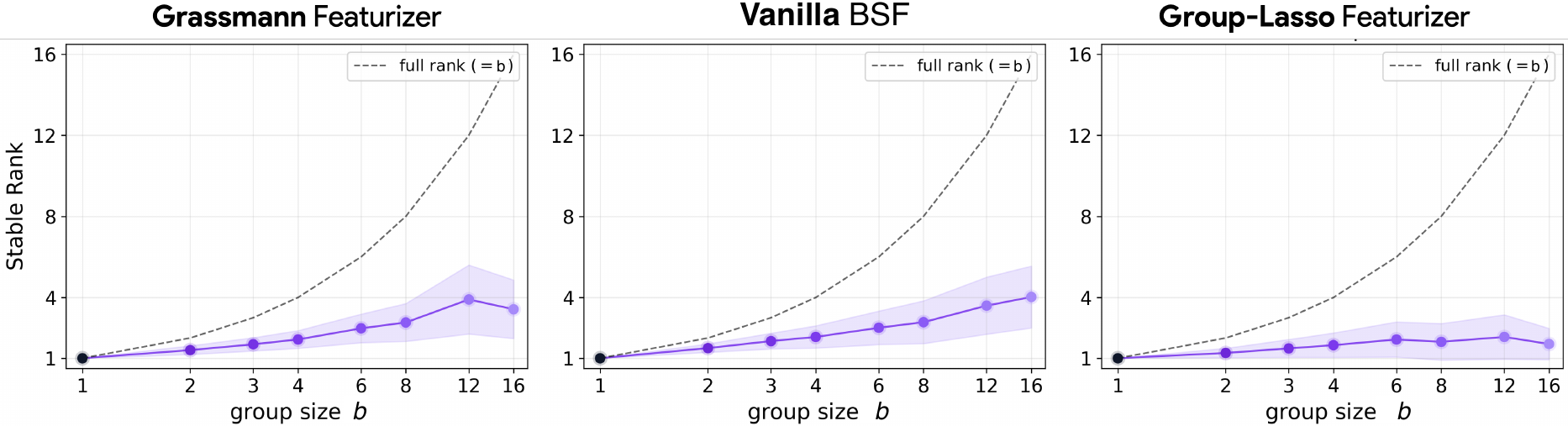}
    \caption{\textbf{Block dimensionality stabilizes between two and four.} Mean stable rank of the per-block code, against the block dimension $\blocksize$ the featurizer was granted. Although blocks are allotted up to $\blocksize=16$ coordinates, the dimension they occupy saturates near $3$, indicating that DINOv3 concepts are on average between two- and four-dimensional regardless of the room made available to them.}
    \label{fig:stable_rank}
\end{figure}

We confine ourselves to a brief account and refer to~\cite{ayonrinde2024interpretability,huang2009structured} and Appendix~\ref{app:mdl} for details.
The principle is the following: a good explanation of a model is one that compresses its activations efficiently, transmitting what occurs in as few bits as possible. We therefore read a sparse code as a compression scheme, carrying the reading from directions to blocks, and ask how many bits are needed to transmit an activation $\bm{x}$ at a chosen distortion $\delta$. 
The distortion is fixed by the task floor estimated above, the fidelity beyond which further reconstruction encodes detail the task does not draw on.
The cost of transmission then decomposes into four parts: which blocks fire, the code they carry, the residual $\bm{r} = \bm{x} - \bm{z}\bm{D}$ they leave unexplained, and the dictionary $\bm{D}$ paid for once across the dataset. Writing $L_\delta(\bm{x})$ for the total number of bits this scheme spends to transmit an activation $\bm{x}$ up to distortion $\delta$, the four parts sum to:
\begin{equation}
\label{eq:mdl}
L_\delta(\bm{x}) =
\underbrace{\log_2\binom{G}{\topk}}_{\text{support } S}
+
\underbrace{L_\delta(\bm{z})}_{\text{code}}
+
\underbrace{L_\delta(\bm{r})}_{\text{residual}}
+
\underbrace{\tfrac{1}{N}\,L(\bm{D})}_{\text{dictionary}}.
\end{equation}

The first term counts which $\topk$ of the $G$ blocks make up the support $S = \operatorname{supp}_{\mathcal{G}}(\bm{z})$, and it is the only term we require in closed form here. The remaining three measure the bits spent on the code $\bm{z}$, on the residual $\bm{r}$ at distortion $\delta$, and on the dictionary $\bm{D}$ amortized over the $N$ tokens it serves, each set out in full in Appendix~\ref{app:mdl}. %

The support term is the most important one, both because it is the largest (scaling with the size of the dictionary) and because it is the term that separates block sparsity from the unstructured sparsity of an ordinary SAE. To state the distinction precisely: an SAE code with $\topk \blocksize$ of its $G\blocksize$ atoms active pays $\log_2\binom{G\blocksize}{\topk \blocksize}$ bits to transmit them, whereas a block-sparse code transmit only its $\topk$ active blocks at the far smaller cost of $\log_2\binom{G}{\topk}$, the $\blocksize$ coordinates inside each active block adding nothing to the index, so that a single block label stands in for $\blocksize$ atoms. %

Enlarging the block dimension therefore trades one cost against another: the residual shrinks, since a larger subspace leaves less unexplained, while the code term grows, since more coordinates must be transmitted per active block. 
Read at a distortion of $\delta = 0.2$, Figure~\ref{fig:mdl_grassmann} shows that the Grassmannian featurizer describe DINOv3 activations in fewer bits than the $\blocksize{=}1$ SAE across every dictionary width and level of sparsity the block codes. The description is shortest at a moderate block dimension, $\blocksize$ from $2$ to $4$, that eases downward as the dictionary widens. This effect is also observed for all three featurizers, at varying noise levels, as reported in Appendix~\ref{app:mdl}.

We have thus seen that, from an MDL standpoint, the inductive bias of block sparsity affords a more compact description of activations. There remains the question of the value of $\blocksize$ itself: what is the optimal block dimension and what is the typical dimensionality of a concept?

\subsection{Concepts Dimensionality}

\begin{figure}[t]
\vspace{-14mm}
\centering
  \includegraphics[width=1\linewidth]{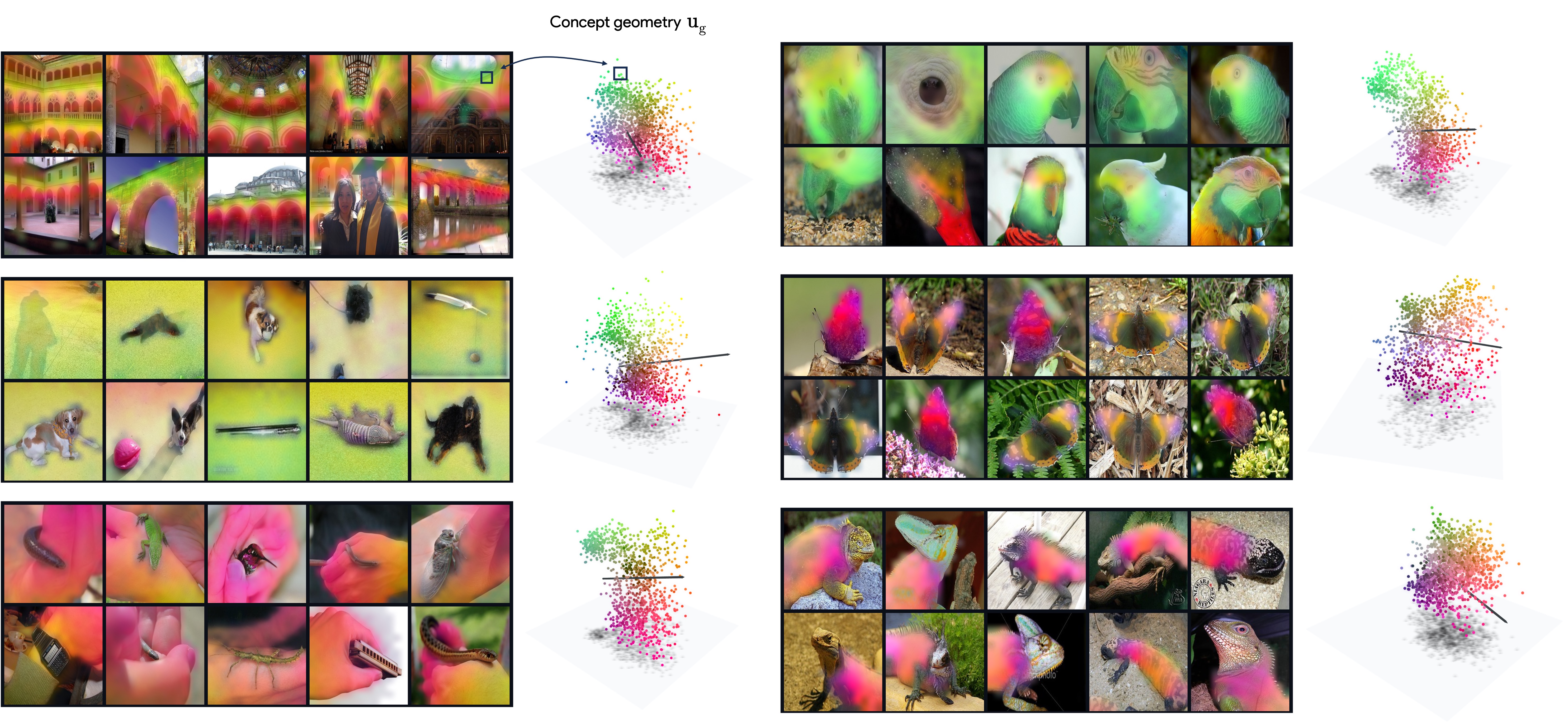}
    \caption{\textbf{A block carries a concept together with its intrinsic geometry.}
    Six blocks from a Grassmann featurizer on DINOv3. For every block we overlay the patches it fires on a sample of images, coloring each patch by the first three principal coordinates of its
    contribution $\bm{u}_g$ mapped to the red, green, and blue channels, so that hue reports where on the concept a patch lies rather than merely how strongly it is present. 
    The black arrow is the closest directional SAE atom contained in the block's subspace, the single direction a $\blocksize{=}1$ code would assign to the concept: it captures one axis of a manifold the block resolves in full, the surrounding spread being exactly the internal variation a directional code discards. 
    }
    \label{fig:qualitative}
    \vspace{-2mm}
\end{figure}

The optimal block dimension found above is a property of the dictionary taken as a whole, and it leaves a more basic question open. When a featurizer grants each atomic unit a $\blocksize$-dimensional block, does the block use all $\blocksize$ dimensions, or does it occupy only a few and leave the rest idle? A block is free to ignore the room it is given, so the dimensions we allocate need not be the dimensions that are used. Suppose a block were used only to produce a single direction, its code $\bm{z}_g$ collapsing to a scalar multiple of one vector, as an SAE atom does. We would then say that, despite the group structure, the block has an effective rank one. A block used to its full extent, by contrast, would describe a $\blocksize$-dimensional subspace, with an effective rank approaching $\blocksize$ when no direction within the block dominates the others.

The stable rank~\citep{rudelson2007sampling} of the per-block code quantifies this, and comparing it against the allotted block dimension reveals how much of the available capacity each block in fact occupies. The mean stable rank saturates between two and four, so that for DINOv3 the average block spans between two and four dimensions, and continues to do so even when the block is allotted substantially greater capacity, up to $\blocksize=16$ (Figure~\ref{fig:stable_rank}).
For DINOv3, we found that blocks are on average roughly three-dimensional, a dimensionality perhaps inherited from the visuospatial nature of the tasks these representations serve.

This places the number in a longer line of work on representational dimensionality. The classical subspace models fixed the dimension of a block by hand---independent subspace analysis, for instance, used subspaces of a preset size~\citep{hyvarinen2000emergence}---whereas the stable rank lets the data choose it, and the value it returns sits within the range those models assumed. It is, moreover, a dimension \emph{per concept} rather than of the population as a whole, complementing measures of the global dimensionality of cortical and model representations~\citep{stringer2019high,fusi2016why}. 

Having completed the quantitative study of these featurizers, we turn to a qualitative one. We first clarify how to read a block-sparse code, and then exhibit several examples of the structures such featurizers recover.

\subsection{How to Interpret a Block Sparse Featurizer}

\begin{figure}[t]
    \centering
    \vspace{-13mm}
    \includegraphics[width=0.99\linewidth]{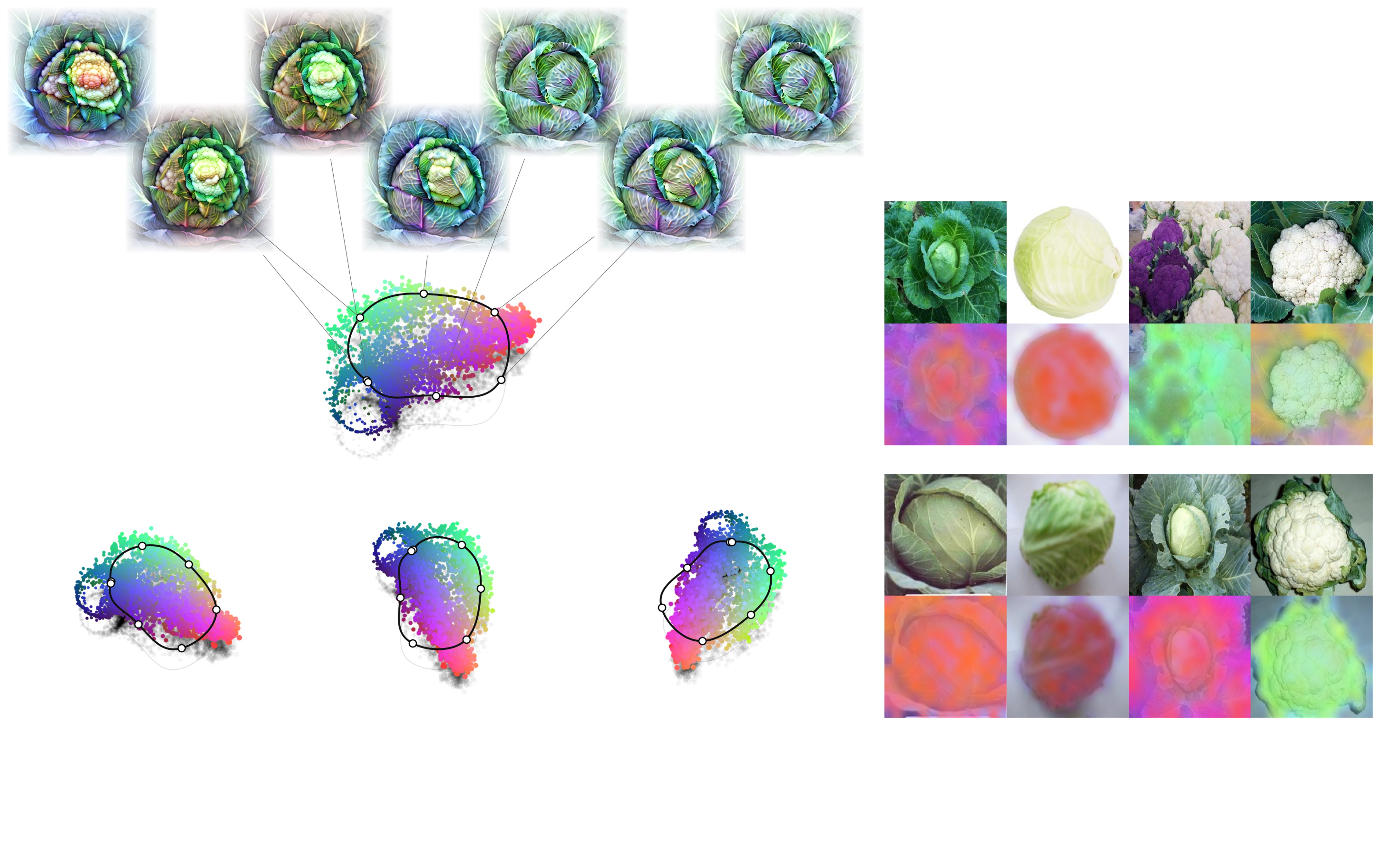}
    \vspace{-12mm}
    \caption{\textbf{Feature visualization on one recovered concept manifold.}
    A single Grassmannian BSF block trained on DINOv3 recovers a cabbage/cauliflower manifold. 
    The central colored point cloud is a projection of the block's active contributions, with hue encoding the first three principal coordinates of the block's intrinsic geometry. 
    White points mark locations sampled along paths through this cloud, and the images above are MACO~\citep{fel2023unlocking} feature visualizations targeted at those locations. 
    As the target moves through the manifold, the synthesized images change smoothly between leafy cabbage structure and cauliflower floret. 
    The natural-image montages on the right ground the same regions in data, showing original patches together with overlays of block activation and manifold coordinate.}
    \label{fig:feature_viz_cauliflower}
    \vspace{-3mm}
\end{figure}

A block-sparse code makes two quantities available per block where a classic SAE code makes only one, and it is worth fixing how each is read before turning to what the featurizers recover. 
We recall that a block fires only when it survives the selection $\bm{\Pi}_{\topk}$, that is when its norm $\|\bm{z}_g\|_2$ ranks among the $\topk$ largest, and this norm is the first of the two quantities: it reports how strongly the concept carried by block $g$ is present, exactly as the activation of an unstructured code does. So the familiar reading of an SAE remains unchanged and a per-token heatmap of $\|\bm{z}_g\|_2$, painted in a single color, gives a coarse view of where a concept is active.

The second quantity has no counterpart in an SAE code, since the contribution of an active block, $\bm{m}_g = \bm{z}_g\bm{D}_g$, ranges over the $k$-dimensional row space of $\bm{D}_g$ rather than along a single fixed direction, so the contributions a block produces across the tokens that fire it trace the intrinsic geometry of the concept, the factors of variation along which it moves once present. 
Reading that geometry asks for a basis, and the gauge invariance discussed before means that the learned chart is defined only up to an $O(\blocksize)$ rotation, no coordinate of $\bm{z}_g$ is privileged over its rotated counterpart. 
Thus, rather than read the arbitrary basis that the featurizer happens to learn, we fit a principal component analysis (PCA) to the contributions $\bm{m}_g$ that a block produces, and we read the code in the resulting basis, which orders the concept's variation by salience. The two readings are then
\begin{equation*}
\begin{gathered}
\text{Concept presence:}\ \ \|\bm{z}_g\|_2
\qquad\qquad
\text{Concept intrinsic geometry:}\ \ 
\bm{u}_g = \bm{z}_g\bm{D}_g\bm{V}_g^\top , \\[4pt]
\text{where}\quad
\bm{V}_g = \argmax_{\bm{V}\bm{V}^\top = \bm{I}_k}\ \big\|\bm{M}_g\bm{V}^\top\big\|_F^2 .
\end{gathered}
\end{equation*}
The basis $\bm{V}_g$ consists of the principal components of the contributions $\bm{M}_g$ that the block produces over the tokens it fires, and the coordinates $\bm{u}_g$ are the code expressed in that basis, their entries being the concept's factors of variation in decreasing order of salience.
Reading a feature thus proceeds at two resolutions, and each maps onto a direct visualization (as shown in Figure~\ref{fig:qualitative}). The block norm $||\bm{z}_g||_2$ reports where a concept is present, rendered as a single-color heatmap over the patches in the manner of an SAE; the coordinates $\bm{u}_g$ report how it varies once active, rendered by assigning the first three principal coordinates to the red, green, and blue channels. Thus, the hue at a patch in the image reads off where on the concept manifold that patch lies as shown in Fig.\ref{fig:intro}.

With our strategy to interpret the featurizers in place, we turn to what they recover, starting in a setting with independently known geometry to validate the method. We then demonstrate a state-of-the-art representation manipulation method where we show that we can improve downstream performance using this BSF basis as well as uncover new types of concepts (shading and shadow). Finally, we will use a generative model to use the recovered manifold as a coordinate system for intervention.

\section{Revisiting How InceptionV1 Processes Curves}
\label{section:inception}

\begin{figure}
\vspace{-12mm}
\centering
\includegraphics[width=0.99\linewidth]{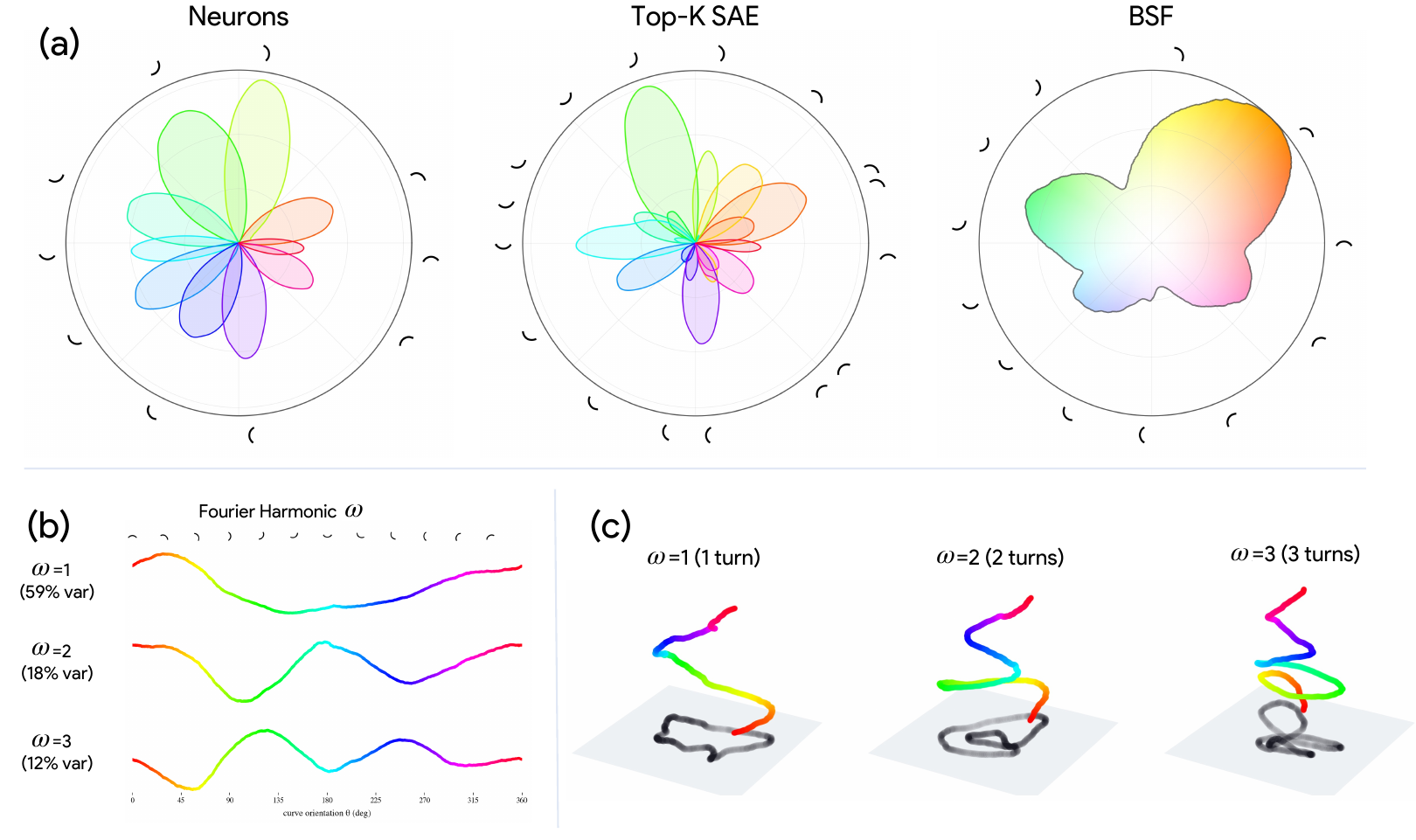}
\caption{\textbf{A single block recovers the InceptionV1 curve manifold, and discovers higher-order Fourier modes.}
\textbf{(a)} Radial tuning curves on synthetic oriented stimuli: individual neurons~\citep{cammarata2020curve} and Top-$k$ SAE atoms~\citep{gorton2024missing} each fire for a narrow wedge of orientations, shattering the family
into many petals, whereas a single BSF block covers all orientations as one connected region. \textbf{(b)} Decomposed along orientation, the block carries more than the orientation circle: its leading components are higher-order Fourier modes, the first three ($\omega,{=}\,1,2,3$) capturing $89\%$ of its variance ($59/18/12\%$), far from random (see Appendix~\ref{app:inception}) with the harmonics beyond the third sharing the remaining $11\%$, none rising near the first three. \textbf{(c)} Projected onto its harmonic-$\omega$ plane, the block traces a loop that winds exactly $\omega$ times.}
\label{fig:curves}
\end{figure}

To evaluate our BSFs against a known ground truth, we begin with a setting whose answer is known in advance. \citet{cammarata2020curve} famously demonstrated that there is a group of neurons in the early layers of InceptionV1 that function as detectors for curves of different orientations. 
Subsequently, this analysis was expanded by~\citet{gorton2024missing} using SAEs. Using BSFs, we show that the neurons and SAE features both provide a fractured view of a unified manifold that represent a curve's orientation.

As a population, both neurons and SAE features tile orientation of curves from $0^\circ$ to $360^\circ$.
The original neurons cover orientation coarsely, leaving ``gaps'', and while SAE atoms fill most of those gaps, they still shatter the manifold~\citep{sengupta2018manifold,michaud2025understanding, bhalla2026sparse}, spending a separate atom on each slice of orientation. This is the predicted optimum of a unstructured prior applied to a curved factor (App.~\ref{app:matched_prior}). 

This multiplicity of SAEs is an artifact of the featurizer---not a property of the network. Under the assumption~\ref{eq:dgp}, the BSFs make the following prediction: a block-sparse code, free to assign a low-dimensional subspace to a single feature, should hold the manifold together where a directional code fragments it. 
A single Grassmannian BSF block does exactly this, staying active across all orientations and reconstructing the whole curve manifold as one connected region (Figure~\ref{fig:curves}a, right): the tiled petals of the directional readings collapse into one feature that carries the orientation circle as its internal geometry.
Recovering the manifold in one block lets us ask how orientation is encoded within it, a question the shattered representation could not pose. We grant the featurizer block of $\blocksize{=}16$ dimensions---far more than a circle requires---and ask how many it uses and to what end.

Sweeping synthetic curves over orientation $\theta$ following \citet{cammarata2020curve}, we record the block's contribution $\bm{m}(\theta)$, a closed curve since $\theta$ is periodic, and decompose it into Fourier harmonics, each harmonic $\omega$ spanning a plane on which $\bm{m}(\theta)$ traces a loop that winds exactly $\omega$ times---the signature of a $\omega$-fold periodic mode (Figure~\ref{fig:curves}c). 
The block spreads its variance across several of these modes, the first three accounting for $89\%$ ($59/18/12\%$): $\omega{=}1$ covers $360^{\circ}$ directional orientation; $\omega{=}2$ is invariant to a $180^\circ$ flip, encoding orientation modulo direction; and $\omega{=}3$ is invariant to $120^\circ$.%
\footnote{Representing a periodic quantity by a small number of Fourier modes is a strategy familiar from classical models of early visual cortex. In the \emph{energy model}~\citep{adelson1985spatiotemporal}, a feature is built by summing the squared outputs of a few filters that share an orientation but differ in phase, so it reports that an oriented pattern is present while staying insensitive to its exact position. Our curve block does the same one level up: it encodes a curve's orientation $\theta$ through a few Fourier modes of $\theta$, rather than the phase of a local edge. An analogous Fourier structure also appears in language models computing modular addition~\citep{kantamnenihelix,zhoufourier,zhou2025fone,fu2026convergent,feucht2026arithmetic}.}
This contextualizes the observation of~\citet{olah2020naturally} that certain families of curve detecting neurons wrap at $180^\circ$ rather than $360^\circ$; these neurons are in fact reading from the second harmonic discovered by our BSFs.
The curve block encodes orientation as an oriented-energy code, much as regions of the brain responsible for visual processing (V1 and V2) do. Higher in the ventral stream, area V4 reads a boundary's curvature together with its position on the object~\citep{pasupathy2001shape,pasupathy2002population}; the same block construction predicts that this richer contour code should appear in deeper layers. Because a sufficiently large Fourier basis can express any representation, we report a null baseline in Appendix~\ref{app:inception}.

\section{Block Sparse Featurizers on a Vision Foundation Model}
\label{section:bsf_dinov3}

\begin{figure}[t]
    \centering
    \includegraphics[width=\linewidth]{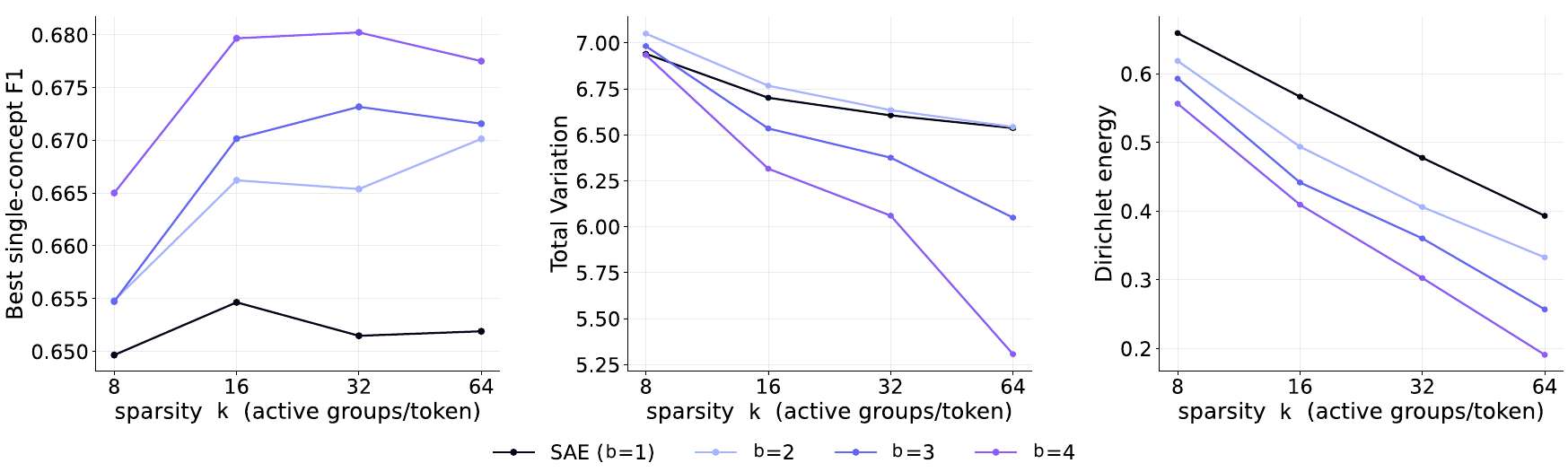}
    \caption{\textbf{Block structure yields class selective, spatially coherent concepts.} Best single-concept detector F1 (left) and concept-map smoothness measured by total variation (center) and Dirichlet energy (right), against sparsity $\topk$ on the native $W{=}32{,}768$ grid; across every $\topk$ the block-SAEs ($\blocksize{>}1$) reach higher F1 and lower total variation and Dirichlet energy than the $\blocksize{=}1$ SAE, the gain widening with $\blocksize$.}
    \label{fig:f1_tv_dirichlet}
\end{figure}

To evaluate our BSFs on a state-of-the-art foundation model, we turn to manifold discovery in the representations of the self-supervised vision model DINOv3~\citep{simeoni2025dinov3}.
Unless noted otherwise, we train a Grassmannian BSF with block size $\blocksize=3$ on final-layer ViT-B patch activations.

\subsection{Class-Selective and Spatially Coherent Concepts}

To test whether the blocks correlate with concepts, we treat the codes as single-concept class detectors. 
We train BSFs of varying block dimension $\blocksize{=}3$ on ImageNet-1k activations, freeze both the backbone and the featurizer, and for each ImageNet-1k class select the single concept whose codes best detect the class. 
We compare against a SAE (the $\blocksize{=}1$ vanilla BSF). 
The best single-concept F1 generally rises with the block sparsity $\topk$, and the BSF codes surpass the SAE at every sparsity with the gap widening as $\blocksize$ increases (Figure~\ref{fig:f1_tv_dirichlet} left). 
This pattern is consistent with blocks that carry finer, more class-selective concepts than the directional codes.

If a block tracks a spatially localized attribute, the map of where it fires across an image should itself be spatially coherent, which we measure through the total variation and Dirichlet energy of the concept maps. Painting each patch with the concepts most active on it and summing the variation over the $4$-connected grid (Appendix~\ref{app:tv}), we find the BSF codes vary less than the SAE codes at matched sparsity, the gap widening as the dictionary grows (Figure~\ref{fig:f1_tv_dirichlet}, center and right); at $\topk{=}64$, the maps are roughly five times smoother at $\blocksize{=}4$ than at $\blocksize{=}1$. The recovered concepts are thus spatially coherent, as one would expect of features that track slowly moving contiguous attributes of a scene, the property the next result makes precise by identifying one such attribute exactly.

\subsection{Discovering a Concept Manifold for a Physical Variable: Lighting and Shadow}

\begin{figure}[t]
    \centering
    \includegraphics[width=1\linewidth]{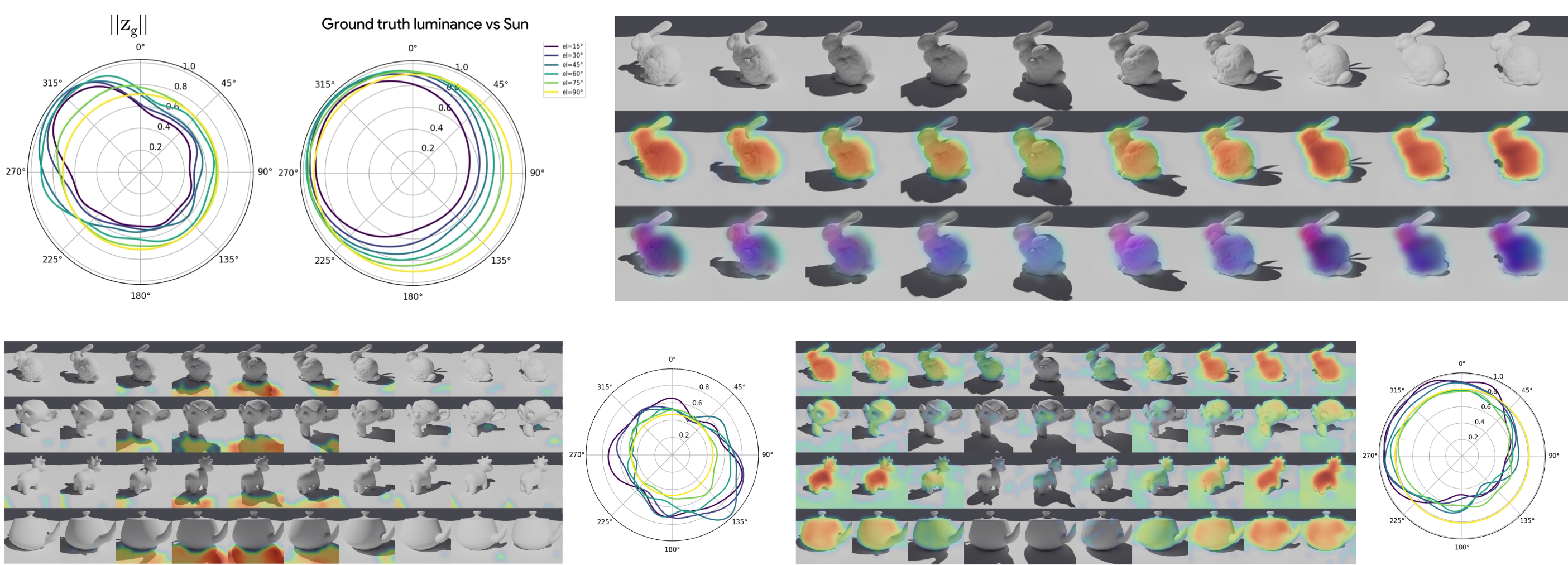}
    \caption{\textbf{Recovered concepts track scene illumination, consistently across objects.} A Grassmannian BSF on DINOv3 is probed on a series of twelve Blender models rendered while the sun sweeps in azimuth and elevation.
    \textbf{(Top left)} The block norm $\|\bm{z}_g\|$ of the luminance block against sun azimuth, one curve per elevation in $\{15^\circ,\dots,90^\circ\}$;
    \textbf{(Top middle)} the ground-truth luminance over the same sweep, which the block norm follows at every elevation;
    \textbf{(Top right)} the recovered luminance concept overlaid on one model across the azimuth sweep, the second row showing concept presence~($\|\bm{z}_g\|$) and the third row its intrinsic geometry~($\bm{u}_g$).
    \textbf{(Bottom)} A recovered shadow concept and the same luminance concept of the top row, overlaid with their azimuthal tuning on four further models (bunny, monkey, cow, teapot); both track illumination regardless of object identity.}
    \label{fig:shadow}
\end{figure}

The spatial coherence of the maps suggests that some blocks may track physical attributes of a scene rather than the objects in it.\footnote{Consistent with~\citet{fel2025into}, who report traces of concepts for lighting and shadow in register tokens.} 
We test this possibility by searching for a block containing a manifold capturing lighting and shadows, because this concept can be precisely controlled. 
We render a series of twelve Blender models under controlled illumination, sweeping the sun's azimuth and elevation, and rank blocks by how their contribution $\bm{m}_g$ co-varies with the lighting parameters across the whole series.
A block scoring high must respond to illumination consistently across objects rather than to any single shape. 
Two blocks stand out: one tracks the scene luminance, i.e., the amount of sunlight reaching the object, and one tracks the volume of cast shadow in the image. 
Both blocks fire on every model in the series, so the concept they carry is attached to the lighting and not to the identity of the object lit.

Because the renderer fixes the lighting, we can check the recovered concept against ground truth. 
Traced against the sun's azimuth, the block norm $\|\bm{z}_g\|$ follows the ground truth luminance curve at every elevation (Figure~\ref{fig:shadow}). 
Read at the two resolutions of the interpretation, the block norm localizes the lit and shadowed regions of each image while the internal coordinate orders the degree of illumination. 
The lighting that an SAE would fragment into many isolated directions, one per region and intensity, is captured by a single block in our BSF.

Having found the recovered structure to be useful for downstream prediction and, in this case, identifiable with a variable of the world, we turn finally to a generative model to evaluate whether blocks are effective for control via activation steering.

\section{Steering Diffusion Models}
\label{section:steering_diffusion}

\begin{figure}[t]
\vspace{-16mm}
    \centering
    \includegraphics[width=0.96\linewidth]{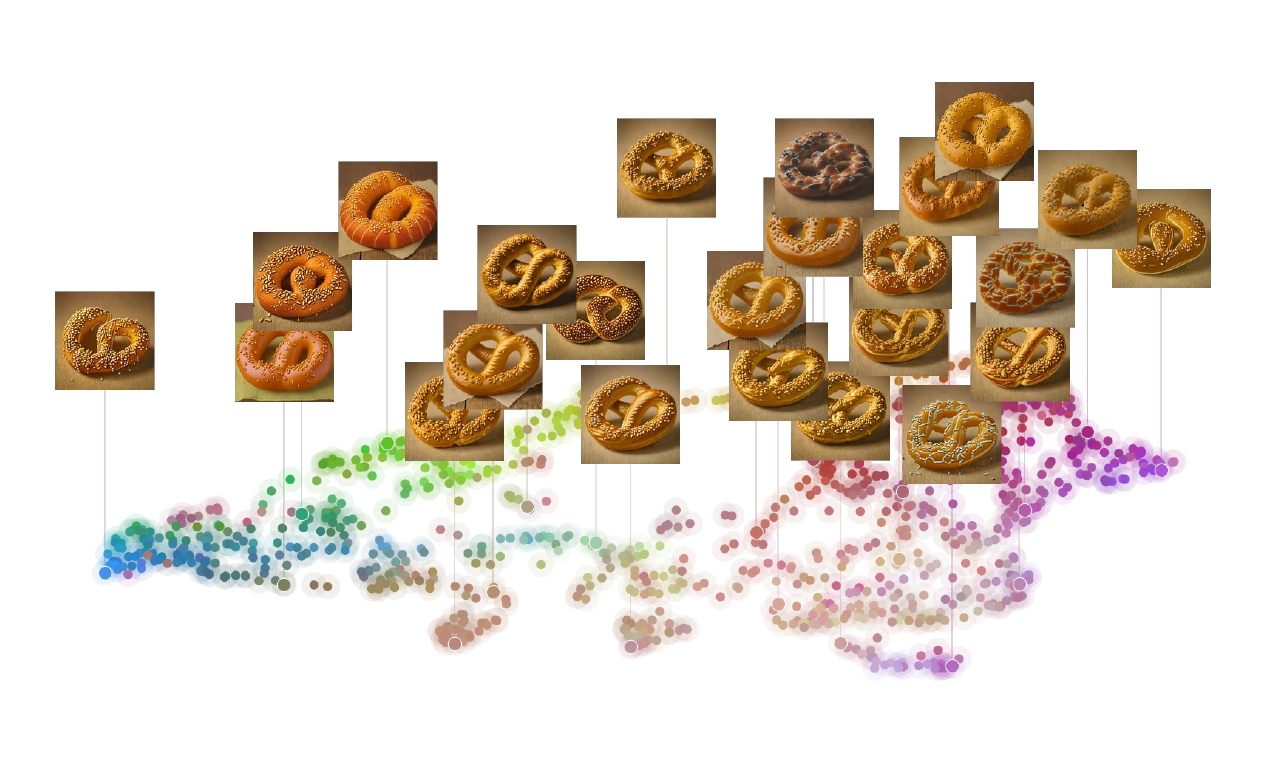}
    \caption{\textbf{Steering the pretzel manifold uncovered in SDXL by BSF.}
    We steer an SDXL block from the component that represents the concept of a pretzel.
    We prompt SDXL for an image of a pretzel and then run four steps of diffusion while fixing the pretzel block to a particular value.
    We visualize the points within a block with UMAP because the $\blocksize = 16$, but color the plot according to the principal components of the block subspace.
    Observe how the images produced by steering reveal the semantics of the pretzel block.
    }
    \label{fig:grid_manifold_steering}
   \vspace{-2mm}
\end{figure}

Up until this point, our analysis has been purely statistical.
However, if BSFs truly capture concepts, then moving within a block should manipulate the behavior according to the geometry of the concept within a block.
To evaluate this, we train a Grassmannian BSF on a diffusion model---specifically SDXL~\citep{podell2024sdxl}---and identify  blocks whose subspace governs an interpretable concept. Then we prompt SDXL to produce an image of that concept and steer within block to generate alternative images. See Figure~\ref{fig:grid_manifold_steering} for steering results on a pretzel manifold that steer to actually realized activations within a block.

In a follow up experiment, we steered along the manifolds within blocks in a more structured manner. 
Activations within a block do not fill a flat plane, but trace a curved low-dimensional manifold (Figure~\ref{fig:grid_steering}), so steering the concept is not a matter of moving along two fixed directions but of following that manifold.
This is uniquely captured by BSF, while the isolated directions of SAEs fail to do so. 
To traverse it in a controlled way we fit a Kohonen map~\citep{guthikonda2005kohonen} (inspired by \cite{konkle2021emergent,doshi2023cortical, yang2026vibe}) to the block manifold and intervene during sampling by setting the block's activations to each grid waypoint and generating one image per waypoint.
The path trace a smooth moves variation in the activations space and the resulting images are displayed in Figure~\ref{fig:grid_steering}. The in-distribution behavior can be computed, although is evident from our assumption (thanks to the group size the variation is large but stay on-manifold in the block subspace).
Put another way, an arbitrary direction in activation space could carry the generation off the manifold of natural images and degrade it, but movement confined to the block's subspace high density regions does not do so, instead maintaining semantic relevance.

\begin{figure}[t]
    \vspace{-12mm}
    \centering
    \includegraphics[width=0.96\linewidth]{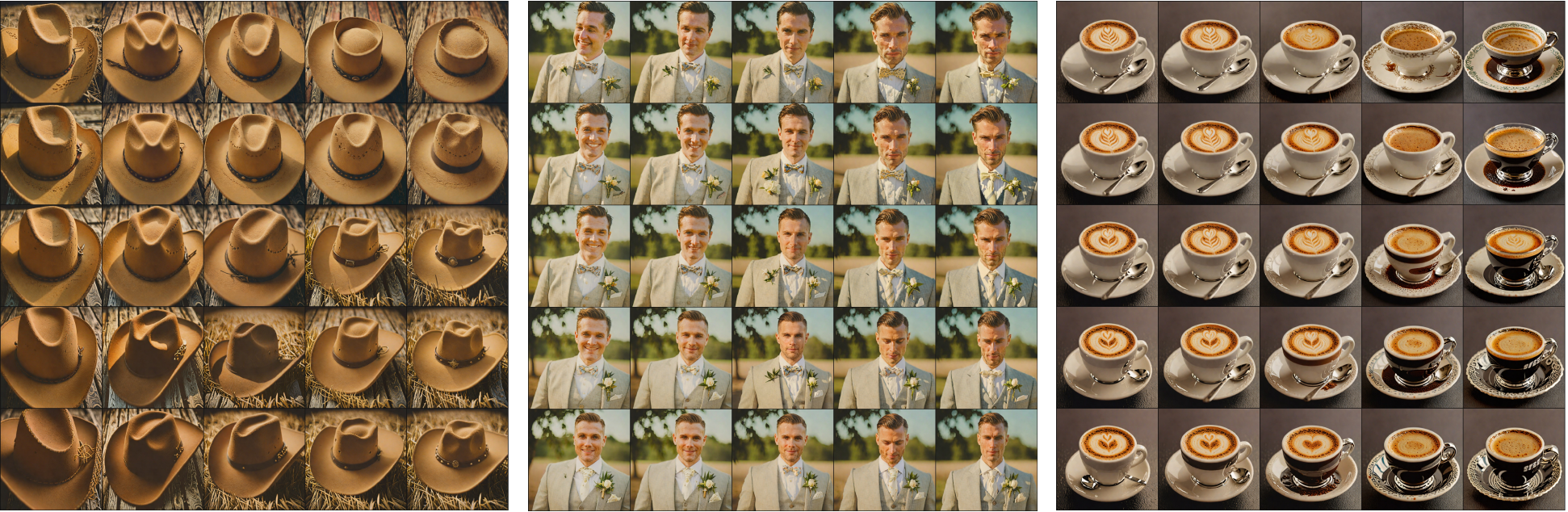}
    \caption{\textbf{Steering blocks in layer~\texttt{down.2.1} of SDXL.}
    Generations produced by setting the block's contribution and fitting a Kohonen map to fit the natural geometry of the 
    subspace (Figure~\ref{fig:grid_manifold_steering}). 
    To produce each grid of images, we prompt SDXL for an image of a hat/man/coffee and then steer across these grid coordinates within the block corresponding to a given concept. Observe that the grid axes aren't interpretable, because this would correspond to two concepts represented by two isolated directions. However, regions of the grid are similar, revealing geometric organization within a block.}
    \label{fig:grid_steering}
    \vspace{-3mm}
\end{figure}

\section{Discussion}
In this work, we posit a data-generating process for neural network representations as a mixture of manifolds and from it derive the principle of block sparsity. 
We then investigate three candidate instantiations of this principle and demonstrate the superiority of BSFs over SAEs, which learn concepts as single directions, across a variety of evaluations.

We can now return briefly to the visual-neuroscience motivation with which we began.
Where an SAE atom is a single direction, a block is a subspace---the distinction the visual cortex draws between simple and complex cells---and it clarifies what reading the block coordinate buys. Cortical models from the energy model~\citep{adelson1985spatiotemporal} to HMAX~\citep{riesenhuber1999hierarchical,serre2007robust} pool over a subspace to discard its internal coordinate and gain invariance: a complex cell signals that an oriented edge is present, not its phase. A block-sparse featurizer instead keeps and reads that coordinate, separating representing a factor from being invariant to it---and that the coordinate carries perceptual content is not in doubt, since visual hyperacuity recovers continuous quantities finer than the detector spacing from the graded activity of a population rather than any single unit~\citep{poggio1992fast}.
This is precisely what is needed if a feature is not a single direction but a low-dimensional manifold.

Useful as this lens is, the mixture of manifolds is a model and not the last word on neural network representations, and we expect that further progress, especially in other modalities, will call for featurizers whose atoms are something other than blocks of directions.
To be precise, we suspect that applying block sparsity directly to language or to video would mismatch the prior~\citep{bhalla2026temporal,lubana2025priors} in the same way a SAE mismatches a mixture of manifolds.
But this limitation is also the point. Blocks are not proposed here as the universal atom of representation, but as an example of what becomes possible when the featurizer is matched to the geometry of the representation it is meant to explain. Where the geometry is hierarchical, temporal, overlapping, or compositional, the matched object should change accordingly.

We therefore end where we began: with the claim that a featurizer is a hypothesis about a data-generating process. The role of interpretability is not only to decompose neural representations into atomic units, but to infer what the atomic units of representations must be. BSFs are one answer to that question and one instance of the broader program of \textit{Neural Geometry}~\citep{geiger2026world}: to let the phenomenology of neural representations tell us what the next answer should be.

\section*{Acknowledgments}

We are grateful to Demba Ba and William Dorrell for fruitful discussions on structured sparsity and dictionary learning that shaped the framing of this work, and to Binxu Wang for helpful discussions on the geometry large vision models. We thank our colleagues at Goodfire, Ren Makino, Emmanuel Ameisen, and the broader interpretability community for valuable feedback at various stages of the project. 

\bibliography{main, manifolds_refs}
\bibliographystyle{bibliostyle}

\newpage
\appendix

\section{Extended Related Work}

\paragraph{Vision interpretability.}
The earliest attempts to explain a vision model relied on attribution methods, that are functionals that given an input and a predictor assign to each input value a scalar that reports its importance to a prediction ~\citep{simonyan2013deep, zeiler2014visualizing, bach2015pixel, springenberg2014striving, smilkov2017smoothgrad, sundararajan2017axiomatic, Selvaraju_2019, Fong_2017,fel2021sobol,muzellec2023gradient}; a subsequent series of papers later established issues and unreliability of those methods \citep{adebayo2018sanity, ghorbani2017interpretation, fel2021cannot,slack2021counterfactual, sturmfels2020visualizing, hsieh2020evaluations, hase2021out,nguyen2021effectiveness}, notably that some of them can survive randomisation of the model weights or of the labels they claim to explain, which denies attribution any claim to report faithfully what the network computes. This motivated the turn to the concept as the unit of explanation, led by~\citet{kim2018interpretability} and developed through methods that recover a concept vocabulary by factorising activations rather than by inspecting individual coordinates, among them ACE \citep{ghorbani2019towards}, ICE \citep{zhang2021invertible}, CRAFT \citep{fel2023craft}, and the work of \citet{vielhaben2023multi}; these were later unified as instances of a single dictionary-learning problem \citep{fel2023holistic}. The sparse autoencoder soon supplied a way to learn such a vocabulary directly from activations under an explicit sparsity constraint \citep{cunningham2023sparse, bricken2023monosemanticity, gao2024scaling, rajamanoharan2024jumping, bussmann2024batchtopk,gorton2024missing}. Subsequent work includes Switch SAE that routes each activation to a single expert sub-dictionary through a learned gate and so trains a much wider dictionary at fixed compute \citep{mudide2025efficient}, while multi-expert variants co-activate several experts and rescale their features to undo the redundancy that single-expert routing induces \citep{xu2025beyond}. Closer to our concern, and as noted in the main paper, two efforts concurrent with ours, SMixAE \citep{francel2026smixae} and a theoretical treatment of subspace-aware coding for language models \citep{dalili2026subspace}, both applied to LLM, relax the assumption that a concept is carried by a single direction, and so point in the direction the present work develops.

\paragraph{Sparse subspace clustering and manifold learning.}
One might reasonably expect manifold learning to be the right apparatus here, and the literature is both large and geometrically sophisticated, offering methods tuned to a wide range of intrinsic geometries through global isometry \citep{tenenbaum2000global, silva2002global}, local linear reconstruction \citep{roweis2000nonlinear, vladymyrov2013locally}, spectral embedding of a neighbourhood graph \citep{belkin2001laplacian, coifman2006diffusion}, and curvature-aware patches \citep{donoho2003hessian, zhang2004principal}; the difficulty is that these methods are built to recover a single manifold, so that they collapse once the representation becomes additive in the sense of a Minkowski sum, recovering the joint manifold its summands trace out together while leaving the concept manifolds, the summands themselves, unrecovered. That a similarity-preserving objective should give rise to localized, manifold-tiling receptive fields is itself a known optimality result \citep{sengupta2018manifold}, which anticipates the tiling regime~\citep{bhalla2026sparse} identified, but for a single manifold rather than an additive mixture of several. Sparse subspace clustering \citep{elhamifar2013sparse} inherits an analogous limitation across an equally developed body of work \citep{liu2010robust, liu2012robust, soltanolkotabi2014robust, you2016scalable, li2017structured, abdolali2021beyond}, including its nonlinear extensions through locality preservation \citep{elhamifar2011sparse}, kernels \citep{patel2014kernel}, deep networks \citep{ji2017deep, li2022neural}, and matching pursuits \citep{tschannen2018noisy}, since its affinity construction presumes that each point belongs to exactly one subspace, so that a point lying in the sum of several subspaces breaks the affinity matrix, or at least alters what it can be taken to mean. The interpretability work of \citet{vielhaben2023multi} does not escape this either, and although it is to our knowledge the first to treat a concept as genuinely multi-dimensional, it does so without additivity, assigning each token to a single subspace rather than to a sum of several.
The same single-assignment limitation reappears in probabilistic form in the recent mixture-of-factor-analyzers decomposition of \citet{shafran2026directions} (and in the related SPADE \citep{hindupur2025projecting}), which adopts the multi-dimensional unit we also take. They model each concept as a low-rank subspace identified up to rotation and the activation set as a mixture of local Gaussian regions to which each token is assigned by a posterior. 
An additive composition is again read as one region rather than factored, and the joint set is covered by local affine charts while the summand manifolds are not targeted (as explained in \citep{bhalla2026sparse}, Appendix C.1).

\paragraph{Group sparsity and visual dictionaries.}
Closest to our own approach is a body of work on structured and group sparsity \citep{jenatton2010structured, bach2012structured, sun2014learning, theodosis2021convergence}, which organizes a code into groups that activate or not together and thereby furnishes the principle on which our featurizer rests. A second and older source lies in visual neuroscience, where learning a dictionary of visual primitives from natural-image statistics has long been standard \citep{olshausen1996emergence, olshausen1997sparse, foldiak2008sparse, serre2006learning}, and where organizing a flat sparse code into blocks recovers the idea that a visual feature is borne by a coordinated group of units rather than by an isolated one, much as a continuous variable is read out from a population of overlapping receptive fields rather than from any single unit \citep{pouget2000information, o1971hippocampus, georgopoulos1986neuronal}. More precisely, the natural neural analogue of a block---a subspace whose activation is the norm of its code---is the complex cell: independent subspace analysis learns exactly such subspaces from natural-image statistics, pooling a group of filters through their energy and thereby acquiring the phase invariance of complex cells~\citep{hyvarinen2000emergence,hyvarinen2001topographic}, a learned form of the classical energy model~\citep{adelson1985spatiotemporal} and later extended to higher-order image structure~\citep{karklin2009emergence,cadieu2012learning}. Where these models pool the within-subspace direction away to obtain invariance, our featurizer retains it, reading the block coordinate as the concept's internal geometry. The same organization reappears at the top of the ventral stream: inferotemporal cortex is arranged into feature columns whose cells share an object feature but vary along an internal dimension, with objects represented as combinations of active columns~\citep{tanaka1991coding,tanaka1996inferotemporal}. Recent work recasts this as a low-dimensional object space whose axes are tuned one at a time~\citep{bao2020map,chang2017code}---a group whose norm signals a feature and whose coordinate parametrizes it, the same structure our block-sparse featurizer recovers in the high-level representations of DINOv3.Taken together, these two lineages converge on block sparsity as the structural primitive. 

\section{Feature Visualization}
\label{sec:feature_visualization}

We visualize recovered concept manifolds with activation maximization~\citep{nguyen2016multifaceted,olah2017feature,nguyen2019understanding}, using specifically the MACO from~\citep{fel2023unlocking} to synthesize images targeted at selected locations of a block manifold. 
For a block $g$, we sample points along paths through the empirical cloud of block contributions and optimize images whose activations align with those targets. 
Figure~\ref{fig:feature_viz_cauliflower} shows the procedure on a single cabbage/cauliflower block: moving across the recovered manifold changes the optimized image to from leafy cabbage to floret. 
Figure~\ref{fig:feature_viz_manifold} extends the same readout to a broader gallery of blocks, showing that the internal coordinates recovered by BSFs expose smooth within-concept variation rather than only whether a concept is present.

\begin{figure}[ht]
    \centering
    \includegraphics[width=0.8\linewidth]{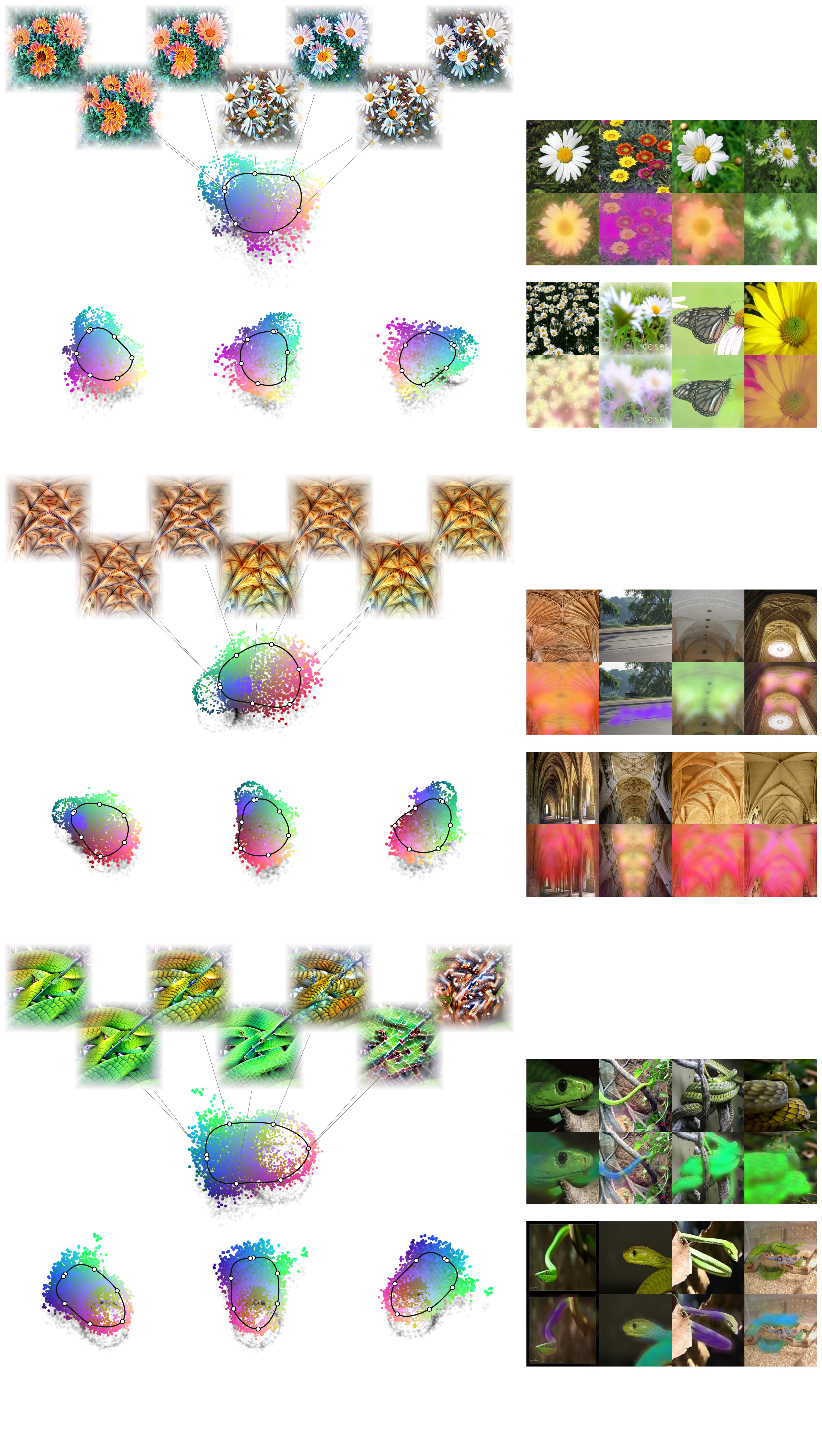}
    \caption{\textbf{A gallery of feature visualizations on recovered concept manifolds.}
    Representative blocks from a Grassmannian BSF trained on DINOv3, visualized with the same convention as Figure~\ref{fig:feature_viz_cauliflower}. 
    For each block, the colored point cloud projects the block's active contributions, with hue encoding the first three principal coordinates of its intrinsic geometry. 
    Feature visualizations sampled along paths through this geometry, shown above each cloud, change smoothly within a coherent visual concept, while the natural-image montages on the right ground the same regions in real activations. 
    The examples span animal fur and markings, boats, cords, bedding, foliage, branches, and repeated manufactured textures, illustrating that a block does not merely detect whether a concept is present, but also provides coordinates for variation within the concept manifold.}
    \label{fig:feature_viz_manifold}
\end{figure}

\begin{figure}[ht]
    \centering
    \includegraphics[width=0.8\linewidth]{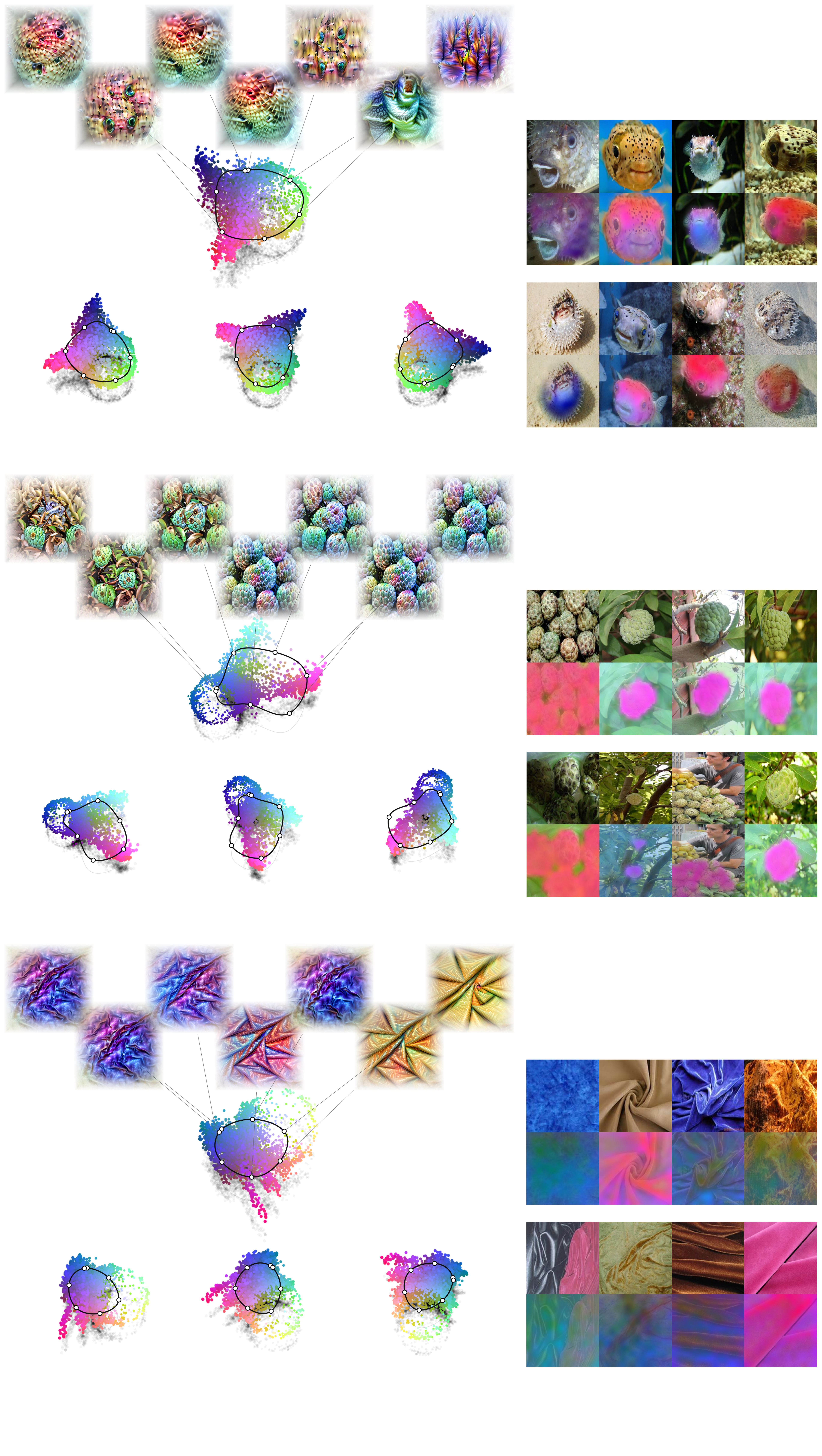}
    \caption{\textbf{A gallery of feature visualizations on recovered concept manifolds.}
    Follow up of Fig.\ref{fig:feature_viz_manifold}.}
    \label{fig:feature_viz_manifold2}
\end{figure}

\clearpage

\section{Minimum Description Length and Feature Dimensionality: Protocol and Full Results}
\label{app:mdl}

This appendix details the controlled sweep behind Section~\ref{subsec:mdl}, the exact estimator of the description length $L_\delta(\bm{x})$ of Eq.~\ref{eq:mdl}, and the per-feature dimensionality study that follows it. Every quantity is read from trained checkpoints and held-out activations, and every figure is regenerated by a committed script with no hand-entered numbers.

\subsection{Sweep and training protocol}

\begin{figure}[t]
    \centering
    \includegraphics[width=\linewidth]{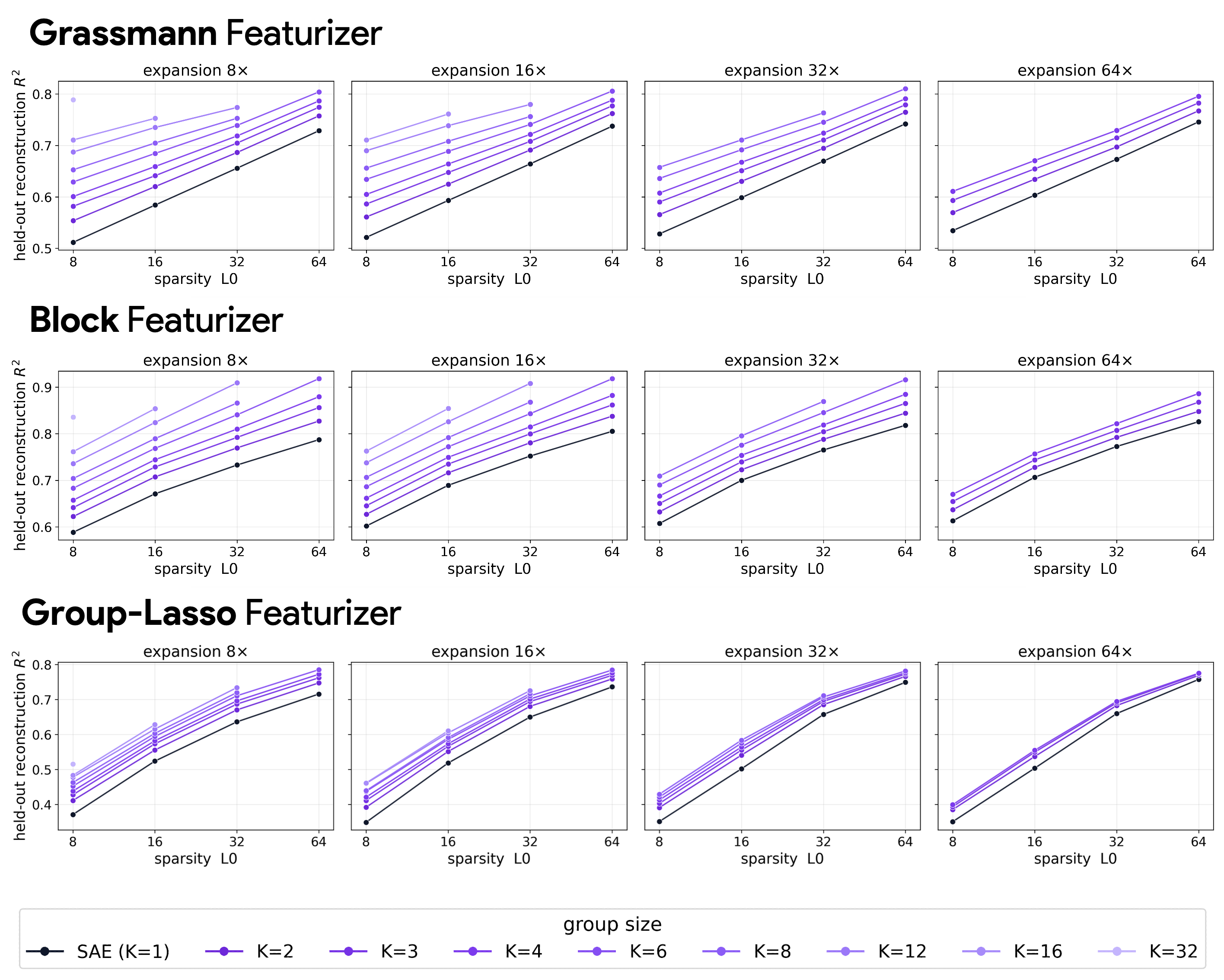}
    \caption{\textbf{Reconstruction improves monotonically with every structural parameter, for all three featurizers.} Held-out $R^2$ against block sparsity $\ell$, one row per featurizer (top to bottom: Block, Grassmannian, Group~Lasso), one column per dictionary width $G$ (expansions $8\times$ to $64\times$), with curves colored by block dimension $k$ ($k{=}1$, in black, is the directional SAE). Quality rises with $\ell$, with $G$, and with $k$ throughout, and the three parameters do not substitute for one another. The advantage of structure over the $k{=}1$ baseline is largest at narrow dictionaries and contracts as $G$ widens, since a wider dictionary can already cover a factor with several small blocks. Because reconstruction is monotone in every direction, it cannot by itself rank the featurizers, which motivates the description-length criterion of Section~\ref{subsec:mdl}.}
    \label{fig:pareto_full}
\end{figure}

We train the three featurizers of Eq.~\ref{eq:bsf} on final-layer DINOv3 ViT-B patch activations ($d = 768$) over a common grid: block dimension $k \in \{1,2,3,4,6,8,12,16,32\}$, dictionary width $G \in \{4096, 8192, 16384, 32768\}$ (expansions $8\times$ to $64\times$), and block sparsity $\ell = \|\bm{z}\|_{2,0} \in \{8,16,32,64\}$. A cell is pruned when it exceeds $Gk > 1.6 \times 10^5$ atoms or $\ell k > 400$ active coordinates, which leaves $96$ realized configurations per featurizer and $288$ in all. The prune keeps every configuration well clear of degeneracy, since $\max \ell k = 384 < d = 768$. All runs share one optimizer schedule, a cosine decay of the learning rate from $10^{-4}$ to $10^{-5}$ with a $2000$-step warmup, three epochs over the activation shards at batch size $8192$, and the input normalization that places $\|\bm{x}\| \approx \sqrt{d}$. 

\subsection{Estimating the description length}

The three terms of Eq.~\ref{eq:mdl} that are not combinatorial rest on three covariance spectra, each harvested per configuration from the featurizer's own forward pass. On a held-out shard, disjoint from the training set, we take the eigenvalues $\lambda_a$ of the covariance of the residual $\bm{r} = \bm{x} - \bm{z}\bm{D}$ together with the held-out $R^2$. On a separate pool of in-distribution activations we take, for each block $g$, its firing rate $p_g$ and the eigenvalues $\sigma^2_{gj}$ of the covariance of its code $\bm{z}_g$ conditional on the block being active. The codes are the featurizer's real gated outputs in each case, the result of $\bm{\Pi}_\ell$ for the Grassmannian and Block featurizers and of the block soft-threshold for the Group~Lasso featurizer, so that for the latter the support size entering the combinatorial term is the measured mean number of active blocks per token rather than a prescribed target.

At a distortion $\delta$, the code and residual terms are the rate-distortion code lengths of these spectra, water-filled against the distortion floor, and the dictionary term amortizes one Grassmann chart over the $N$ tokens it serves. Writing $(\cdot)_+ = \max(\cdot, 0)$,
\begin{equation}
\label{eq:mdl-terms}
L_\delta(\bm{z}) = \sum_{g} p_g \sum_{j=1}^{k} \tfrac12\!\left(\log_2 \tfrac{\sigma^2_{gj}}{\delta}\right)_{\!+},
\qquad
L_\delta(\bm{r}) = \sum_{a} \tfrac12\!\left(\log_2 \tfrac{\lambda_a}{\delta}\right)_{\!+},
\qquad
\tfrac1N L(\bm{D}) = \tfrac12 \tfrac{G\,k\,(d-k)}{N} \log_2 N,
\end{equation}
the dictionary count following from the $k(d-k)$ free parameters of a point on the Grassmannian $\mathrm{Gr}(k, d)$. The support term is the exact combinatorial $\log_2\binom{G}{\ell}$, and the total is $L_\delta(\bm{x}) = \log_2\binom{G}{\ell} + L_\delta(\bm{z}) + L_\delta(\bm{r}) + \tfrac1N L(\bm{D})$ of Eq.~\ref{eq:mdl}.

Distortion is measured relative to a reference variance common to every configuration, $\delta = \mathrm{frac} \cdot \mathrm{var}_{\mathrm{ref}}$ with $\mathrm{frac} = 1 - R^2_{\mathrm{target}}$, so that the four headline levels $\mathrm{frac} \in \{0.01, 0.05, 0.10, 0.20\}$ correspond to $R^2_{\mathrm{target}} \in \{0.99, 0.95, 0.90, 0.80\}$. The $20\%$ level coincides with the DINOv3 noise floor estimated in Figure~\ref{fig:pareto}, and is the level at which the main text reads the result. We evaluate each water-filling term with the smooth Gaussian rate $\tfrac12 \log_2(1 + \sigma^2/\delta)$, which agrees with the clipped form wherever $\sigma^2 \gg \delta$ but removes the discontinuity an eigenvalue would otherwise introduce as it crosses $\delta$, leaving the description-optimal block dimension steadier as $\delta$ is varied.

\begin{figure}[h]
    \centering
    \includegraphics[width=0.9\linewidth]{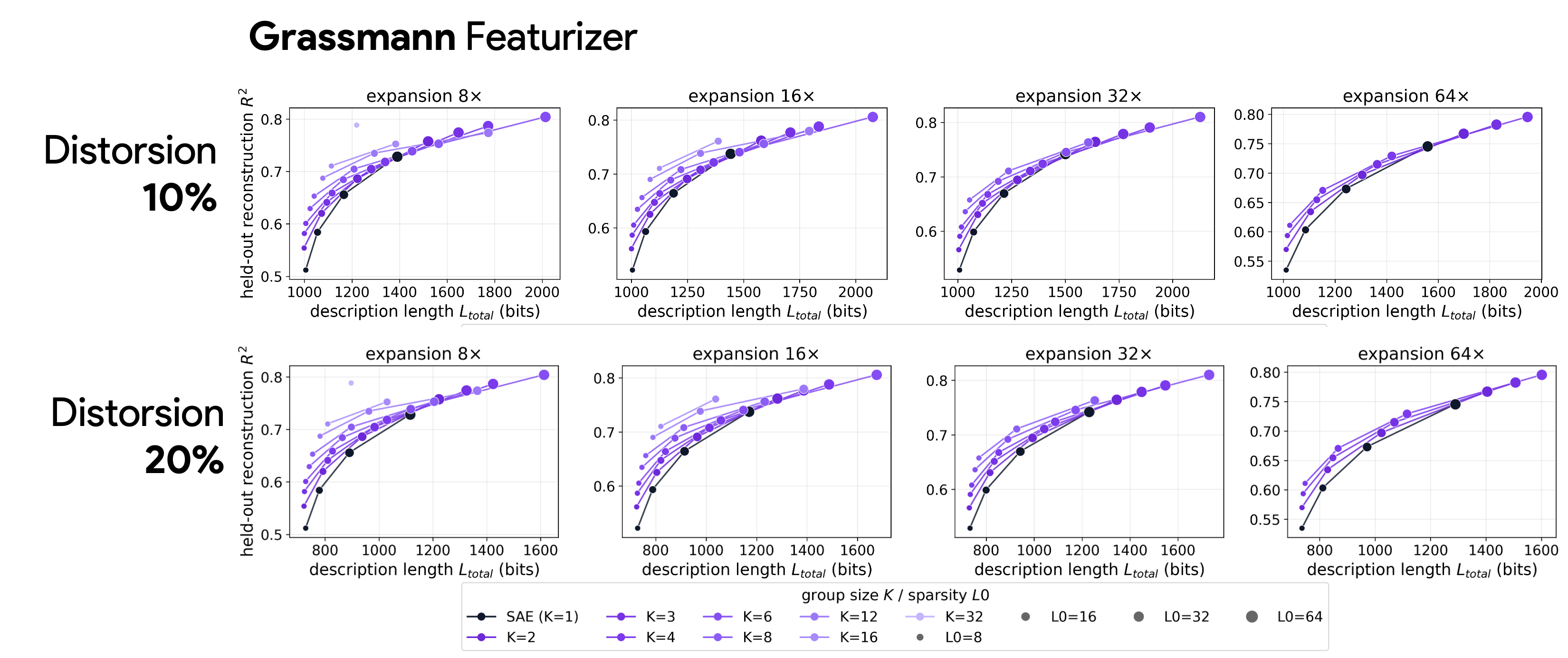}
    \caption{\textbf{Grassmannian featurizer, full description-length landscape.} Description length $L_\delta(\bm{x})$ (Eq.~\ref{eq:mdl}) against block dimension $k$, with one panel per dictionary width $G$ and one curve per block sparsity $\ell$, at each distortion level. The minimum of each curve marks the description-optimal block dimension, which sits at an interior $k \approx 3$ and eases downward as $G$ widens.}
    \label{app:mdl:grassmann}
\end{figure}

\begin{figure}[h]
    \centering
    \includegraphics[width=0.9\linewidth]{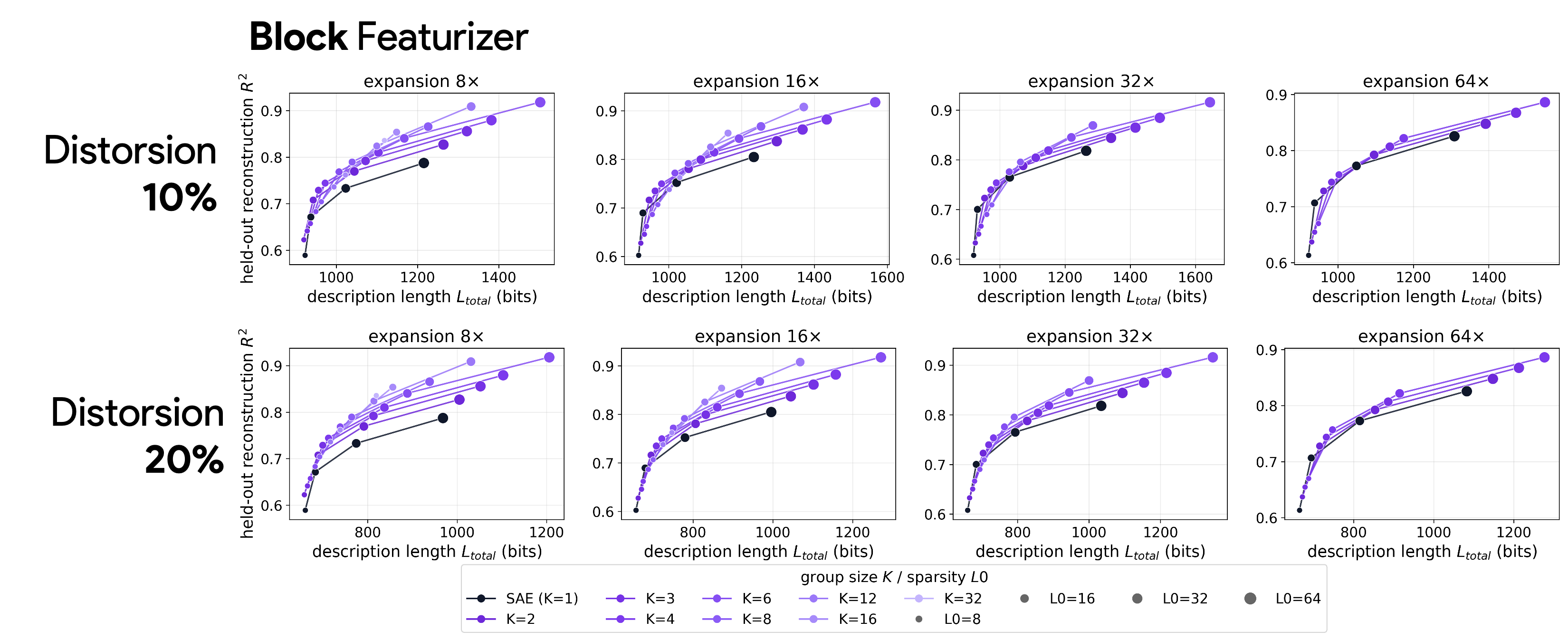}
    \caption{\textbf{Block featurizer, full description-length landscape.} As in Figure~\ref{app:mdl:grassmann}, for the free-decoder Block featurizer. The structured codes again describe activations in fewer bits than the $k{=}1$ SAE, though the optimum sits nearer the directional limit, $k \in \{1,2,3\}$, since the untied dictionary already distributes structure across atoms.}
    \label{app:mdl:block}
\end{figure}

\begin{figure}[h]
    \centering
    \includegraphics[width=0.9\linewidth]{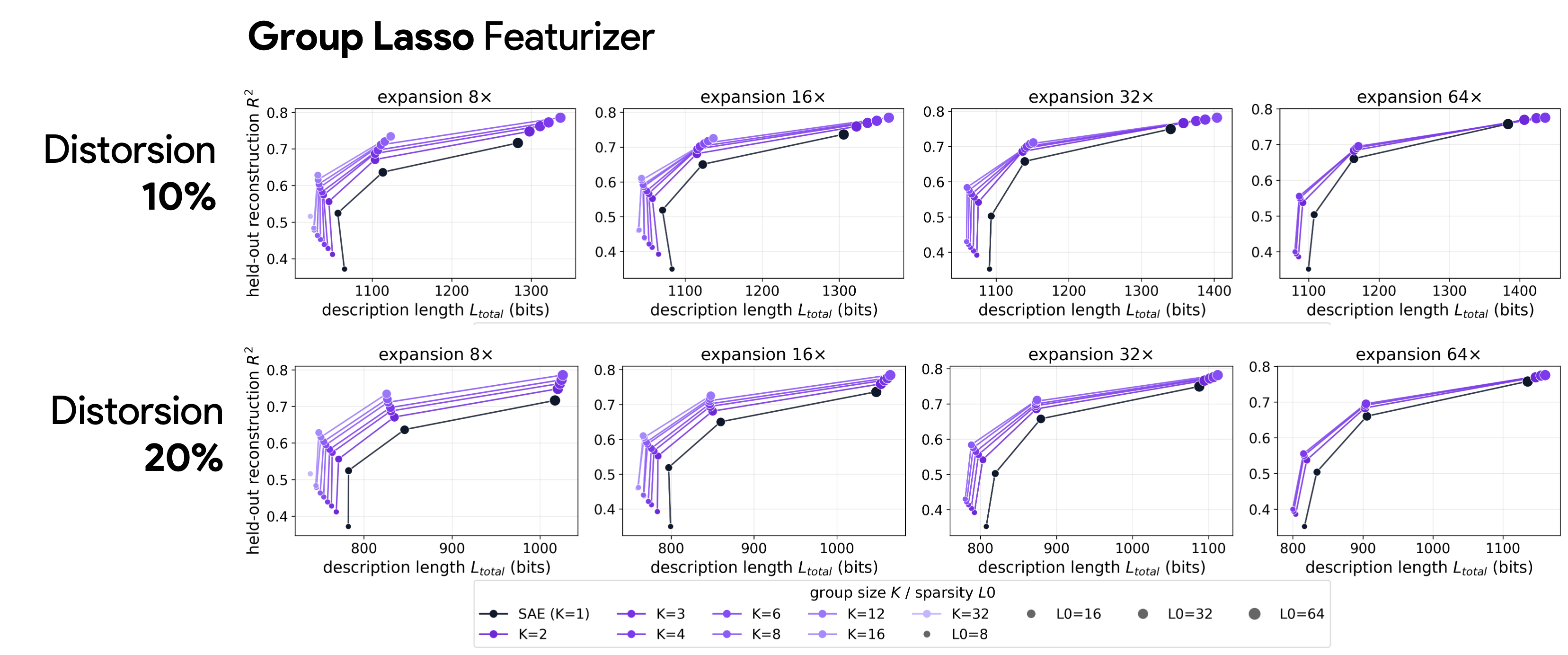}
    \caption{\textbf{Group~Lasso featurizer, full description-length landscape.} As in Figure~\ref{app:mdl:grassmann}, for the Group~Lasso featurizer, compared at matched average $\|\bm{z}\|_{2,0}$ since its sparsity is not fixed in advance. It saves the most support bits and tolerates the largest blocks, yet realizes them at a higher residual cost, in keeping with the low per-block utilization documented in Figure~\ref{fig:rank_appendix}.}
    \label{app:mdl:glasso}
\end{figure}

\subsection{Description length favours structure at every distortion}

At the $20\%$ noise floor the Grassmannian featurizer attains its shortest description at a block dimension between $2$ and $4$, as reported in the main text (Figure~\ref{fig:mdl_grassmann}). The conclusion holds across the other featurizers and the remaining distortion levels: the structured codes describe DINOv3 activations in fewer bits than the directional code to which they reduce at $k{=}1$, and the description-optimal block dimension is moderate and eases downward as the dictionary widens. The featurizers differ in where the optimum falls. The Grassmannian featurizer, whose orthonormal charts render every block dimension fully effective, shows the clearest interior optimum near $k \approx 3$. The Block featurizer, its untied dictionary already free to spread structure across atoms, settles nearer the directional limit at $k \in \{1,2,3\}$. The Group~Lasso featurizer tolerates the largest blocks but realizes them at a higher residual cost, in keeping with the lower per-block dimensionality documented below. As these optima are read from shallow minima on a coarse grid, we take the direction of the effect, and not any single value of $k$, to be the reliable conclusion.

\subsection{Feature dimensionality: how much of a block is used}

The description-length analysis selects a block dimension for the dictionary as a whole, and the per-feature study asks a complementary question: of the $k$ dimensions a featurizer grants each block, how many does a feature actually occupy? We measure this with three rank statistics of the per-block code covariance, the stable rank $\|\cdot\|_F^2 / \|\cdot\|_2^2$, the participation ratio, and the effective rank (the exponential of the spectral entropy), each equal to $1$ for a block that carries a single direction and approaching $k$ for a block whose $k$ dimensions are used evenly. We summarize occupancy by the utilization, the effective rank divided by the allotted $k$, which is $1$ for a fully used block and falls toward $0$ as the block leaves dimensions idle.

\begin{figure}[t]
    \centering
    \includegraphics[width=\linewidth]{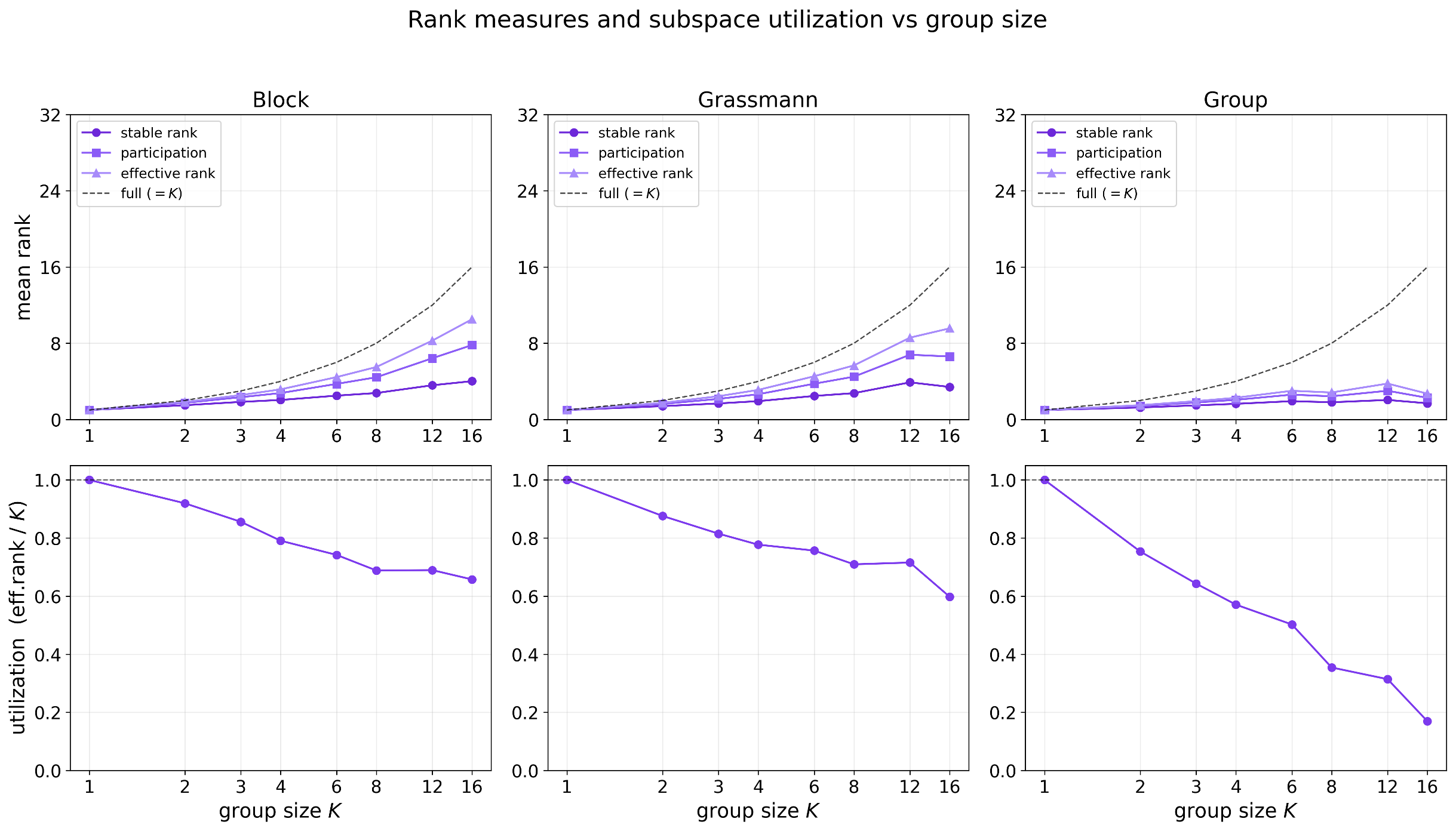}
    \caption{\textbf{Rank measures and utilization against block dimension.} \textbf{(top)} Mean stable rank, participation ratio, and effective rank of the per-block code, against the allotted block dimension $k$, for the three featurizers; the dashed line is full occupancy ($\mathrm{rank} = k$). All three measures grow far more slowly than the diagonal, so a block uses only a fraction of the dimensions it is given. \textbf{(bottom)} Utilization, the effective rank as a fraction of $k$, against $k$. Utilization falls from $1$ throughout, mildly for the Block and Grassmannian featurizers (to roughly $0.6$--$0.66$ at $k{=}16$) and steeply for the Group~Lasso featurizer (to about $0.17$ at $k{=}16$), whose convex shrinkage concentrates each block on a few directions.}
    \label{fig:rank_appendix}
\end{figure}

Figure~\ref{fig:rank_appendix} reports the three measures and the utilization across the grid. All three ranks grow with $k$ far more slowly than the full-rank diagonal, so a block uses only part of the dimensions it is allotted, and the gap widens as $k$ grows. The Block and Grassmannian featurizers keep utilization high, falling from $1$ to roughly two-thirds at $k{=}16$, consistent with blocks that genuinely fill a low-dimensional subspace. The Group~Lasso featurizer is different: its utilization collapses toward $0.17$ at $k{=}16$, so the blocks it nominally widens are in fact spent on a handful of directions, which is the per-feature counterpart of the higher residual cost it pays in the description-length analysis.

Read together, the description-length and per-feature analyses give one picture. Structure shortens the description of DINOv3 activations at a moderate block dimension, and the dimension a feature actually occupies is moderate for the same reason, with the average stable rank settling between two and four across featurizers (Figure~\ref{fig:stable_rank}) even when blocks are allotted up to $k{=}16$. The block dimension that the dictionary prefers and the dimension that a feature fills are two readings of the same fact, that DINOv3 features are on average roughly three-dimensional.

\section{Implementation of the Block-Sparse Featurizers}
\label{app:implementations}

All three featurizers share the same skeleton: a linear encoder that produce group of signed codes (positive or negative, no more \texttt{ReLU} involved) and a linear decoder. Activations are scaled such that $\mathbb{E}\big(||\bm{x}||_2\big) = \sqrt{d}$ before encoding, the decoder is linear, $\hat{\bm{x}} = \bm{z}\bm{D}$ with $\bm{D} \in \mathbb{R}^{Gk \times d}$ stacking the blocks $\bm{D}_g \in \mathbb{R}^{k \times d}$, and training minimizes reconstruction error over the dataset with Adam. They differ in the encoder, the constraint set, and the loss, which we now give per method.

\paragraph{Grassmannian BSF.}
The objective is reconstruction with orthonormal charts and a tied encoder,
\begin{equation*}
\min_{\bm{D}, \gamma} \; \mathbb{E}\,\|\bm{x} - \bm{z}\bm{D}\|_2^2
\quad \text{s.t.} \quad \bm{D}_g \bm{D}_g^\top = \bm{I}_k \;\; \forall g.
\end{equation*}
Its parameters are the charts $\bm{D}_g$ and a single positive scalar $\gamma$, shared across blocks. The code is built block-wise,
\begin{equation*}
\bm{z}_g = \gamma \, \bm{x} \bm{D}_g^\top \in \mathbb{R}^k,
\qquad
\bm{z} = \bm{\Pi}_\ell\big(\bm{z}_1, \dots, \bm{z}_G\big),
\end{equation*}
where $\bm{\Pi}_\ell$ retains $\bm{z}_g$ if $g$ is among the $\ell$ blocks of largest norm and zeroes it otherwise. The loss is pure reconstruction, sparsity being enforced by $\bm{\Pi}_\ell$ and the constraint by re-projecting each $\bm{D}_g$ onto the Stiefel manifold via QR decomposition. In practice we found the re-projection need not occur at every step: blocks drift slowly from orthonormality, and projecting every 20 steps suffices. The scalar $\gamma$ compensates the energy lost by tying encoder and decoder; we parameterize it through its logarithm to keep it positive and initialize it so reconstruction of the average activation is unbiased at step zero. Selection is per-sample; batch-level variants in the spirit of BatchTopK~\citep{bussmann2024batchtopk} are a natural extension we leave to future work.

\paragraph{Vanilla BSF.}
Both maps are now free. The parameters are an encoder $(\bm{W}, \bm{b})$ with $\bm{W} \in \mathbb{R}^{d \times Gk}$ and the dictionary $\bm{D}$, with each decoder block constrained to the unit ball, $\|\bm{D}_g\|_F \leq 1$, and the objective is reconstruction under the code constraint $\|\bm{z}\|_{2,0} \leq \ell$. The code is
\begin{equation*}
\bm{z}_g = (\bm{x}\bm{W} + \bm{b})_g \in \mathbb{R}^k,
\qquad
\bm{z} = \bm{\Pi}_\ell\big(\bm{z}_1, \dots, \bm{z}_G\big).
\end{equation*}
The loss is again pure reconstruction. The ball constraint resolves the scale ambiguity between $\bm{z}$ and $\bm{D}$, which would otherwise let block norms drift and corrupt the top-$\ell$ selection. We initialize $\bm{W} = \bm{D}^\top$ with the encoder scaled so pre-code block norms have the correct order of magnitude at initialization, which removes most dead blocks without auxiliary losses.

\paragraph{Group Lasso BSF.}
The hard constraint is replaced by its convex surrogate. The parameters are $(\bm{W}, \bm{b}, \bm{\theta}, \bm{D})$ with one learned threshold $\theta_g > 0$ per block, and the code applies the block shrinkage,
\begin{equation*}
\bm{z}_g = \operatorname{sh}_{\theta_g}\big((\bm{x}\bm{W} + \bm{b})_g\big),
\qquad
\bm{z} = (\bm{z}_1, \dots, \bm{z}_G),
\end{equation*}
where $\operatorname{sh}_{\theta}(\bm{u}) = \big(1 - \tfrac{\theta}{\|\bm{u}\|_2}\big)_{+} \bm{u}$ is the proximal operator of the $\ell_{2,1}$ norm, so a forward pass is one step of proximal gradient on the relaxed objective~\citep{gregor2010learning}. The featurizer specific loss is
\begin{equation*}
\mathbb{E}\,\|\bm{x} - \bm{z}\bm{D}\|_2^2 + \lambda \|\bm{z}\|_{2,1}.
\end{equation*}
In practice, to target a point on the sparsity--reconstruction Pareto front, we apply the $\ell_{2,1}$ term only when the block sparsity exceeds a target threshold; this gives better control over the resulting $\|\bm{z}\|_{2,0}$ than tuning $\lambda$ alone, though we regard the schedule as a practical device and a more principled treatment as future work. Since the resulting sparsity is not fixed in advance, comparisons against the other two featurizers are made at matched average $\|\bm{z}\|_{2,0}$.

\paragraph{Losses.}
All three featurizers minimize the same objective,
\begin{equation*}
\mathcal{L} = \underbrace{\mathbb{E}\,\|\bm{x} - \bm{z}\bm{D}\|_2^2}_{\text{reconstruction}}
\;+\; \underbrace{\lambda\,\mathbb{E}\,\|\bm{z}\|_{2,1}}_{\text{Group Lasso only}}
\;+\; \underbrace{\alpha\,\mathbb{E}\,\|\bm{r} - \bm{z}'\bm{D}\|_2^2}_{\text{AuxK}},
\end{equation*}
where $\lambda = 0$ for the Grassmannian and Block BSFs. The last term is the block analogue of the auxiliary loss of~\citet{gao2024scaling}: $\bm{r} = \bm{x} - \bm{z}\bm{D}$ is the residual and $\bm{z}' = \bm{\Pi}_\ell(\bm{z}_1, \dots, \bm{z}_G)$ is the code built from the next $\ell$ blocks by norm among those not selected, so the runner-up blocks are asked to explain what the selected ones missed. %
We use $\alpha = 1/\ell$ throughout.

\section{Revisiting InceptionV1 Curves.}
\label{app:inception}

A Fourier probe of a $K{=}16$ subspace could in principle report power at modes
$k\!\ge\!2$ that arises from noise, or from the subspace and the featurizer,
rather than from the signal, since a Fourier decomposition is a complete
orthogonal basis that reconstructs any response on the circle exactly given
enough components, so that power at the higher modes is something the basis can
always supply and is not on its own evidence of periodic structure in the
signal. This raises the question of how the modes the probe recovers can be
shown to be real rather than artifacts of the basis: they are real in one sense
already, in that the probe isolates directions a downstream layer or task could
read off the representation, though legibility of this kind does not settle the
matter, since a recovered direction counts only when the variance it explains,
the norm of the response once projected onto it, is quantitatively large rather
than negligible. We rule the artifact reading out with a matched null
(Table~\ref{tab:fourier_null}).

Writing $\bm{m}(\theta)$ for the recovered curve group's reconstruction
contribution as the stimulus orientation $\theta$ sweeps $[0,360^\circ)$, we split
the group's input response into its first harmonic (the orientation circle,
$k{=}1$) and a residual, replace that residual with white-in-$\theta$ Gaussian
noise of equal off-circle energy, and pass this orientation-plus-noise input
through the same trained group and identical probe. The group is active at every
orientation (coverage $1.0$), so its gated code is a linear function of its input
and cannot create harmonics by gating: a noiseless $k{=}1$ input returns a response
that is purely first harmonic, and matched noise ($56\%$ of the input variance)
reaches only the analytic white-noise floor $1/(n_\theta/2)=0.06\%$ per mode at
$n_\theta{=}1800$. The measured manifold instead places $18.3\%$ and $12.0\%$ of
its variance at $k{=}2$ and $k{=}3$, exceeding the null by factors of $470$ and
$314$ ($z\approx1.5\times10^{3}$ and $9\times10^{2}$ over $200$ draws). The second
and third harmonics are therefore genuine periodic structure in the curve
response, inherited from the network, whose raw response carries even more of it,
and not an artifact of noise, the subspace, or the SAE.
\begin{table}[h]
\centering
\small
\caption{\textbf{Null test for the higher harmonics.} Percentage of AC variance of
$\bm{m}(\theta)$ in each Fourier mode $k$. A pure first-harmonic input, and a
first-harmonic input plus white-in-$\theta$ noise carrying the same off-circle
energy, are passed through the same $K{=}16$ curve group. Both nulls leave
$k\!\ge\!2$ at the white-noise floor; the measured manifold exceeds it by two
orders of magnitude, so modes $k{=}2,3$ are real structure inherited from the
network.}
\label{tab:fourier_null}
\begin{tabular}{lcccc}
\toprule
 & $k{=}1$ & $k{=}2$ & $k{=}3$ & $k{=}4$ \\
\midrule
raw curve response                          & $43.6$  & $25.9$ & $16.8$ & $7.4$ \\
recovered manifold $\bm{m}(\theta)$          & $58.9$  & $18.3$ & $12.0$ & $7.2$ \\
null: orientation only ($k{=}1$, no noise)   & $100.0$ & $0.0$  & $0.0$  & $0.0$ \\
null: orientation $+$ matched noise          & $65.6$  & $0.04$ & $0.04$ & $0.04$ \\
\bottomrule
\end{tabular}
\end{table}

\begin{figure}
    \centering
    \includegraphics[width=0.85\linewidth]{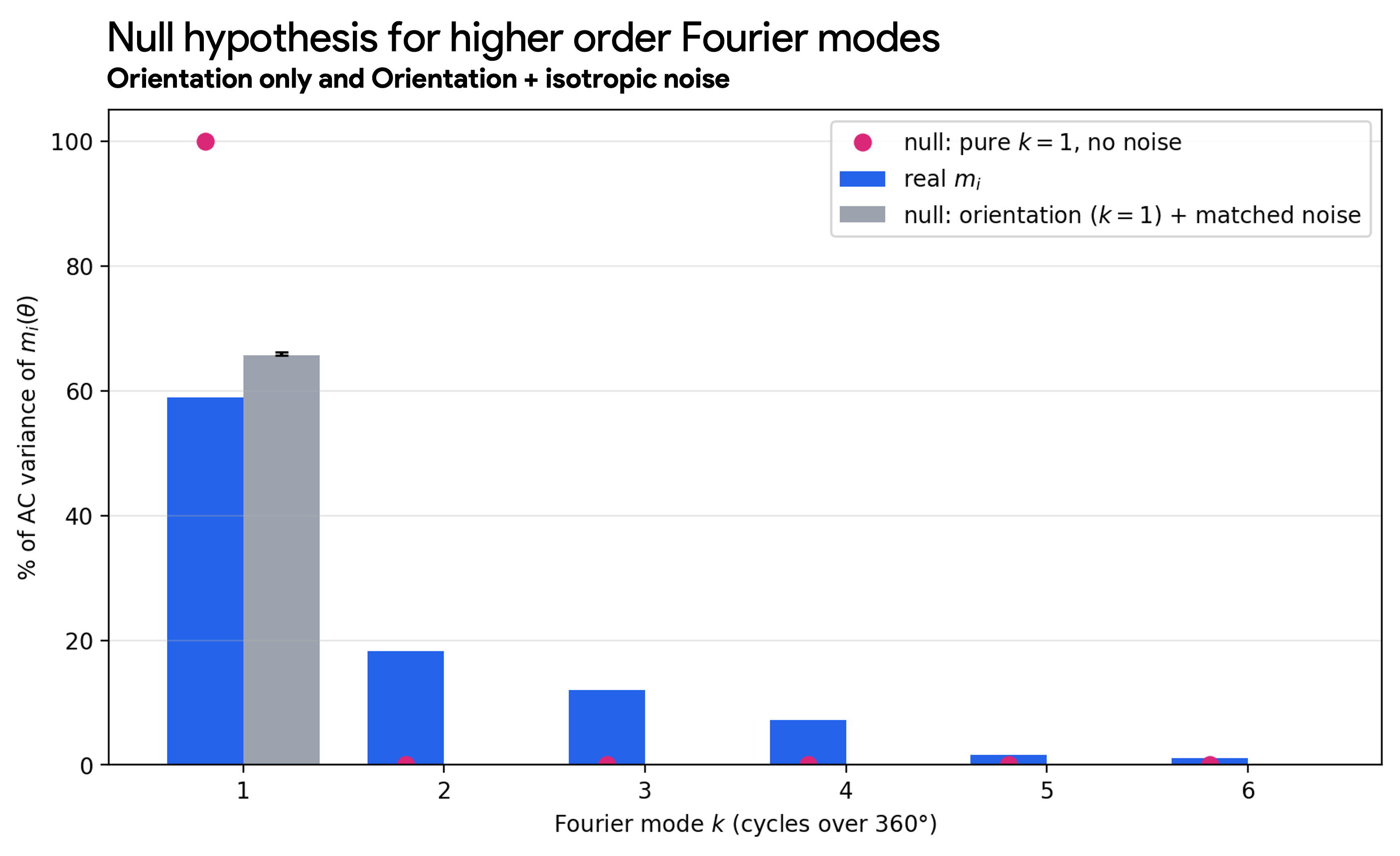}
    \caption{\textbf{The higher harmonics survive a matched null.} Share of the AC
variance of the recovered curve manifold $\bm{m}_i(\theta)$ carried by each
Fourier mode $k$ as the stimulus orientation sweeps $[0,360^\circ)$. The measured
manifold (blue) retains $18.3\%$ and $12.0\%$ of its variance at $k{=}2$ and
$k{=}3$, while the two nulls leave these modes empty: a pure first-harmonic input
(pink, $k{=}1$ only) places all of its power on the orientation circle, and the
same input perturbed by white-in-$\theta$ noise of matched off-circle energy
(grey) sits at the analytic white-noise floor for every $k\!\ge\!2$. The second
and third harmonics are therefore periodic structure inherited from the network
rather than an artifact of noise, the subspace, or the SAE.}
\label{fig:fourier_null}
\end{figure}

\section{Group Sparsity as Matched prior}
\label{app:matched_prior}

This appendix proves Proposition~\ref{lemma:matched_prior} and collects some surrounding remarks. 
We recall our claim here, starting from the data generating process defined in Def. \ref{def:amm}, and our additional assumption introduced in the main text: for each factor
$\mathcal{M}_g$, let $\bm{V}_g=\operatorname{span}(\mathcal{M}_g)$ and assume
$\dim \bm{V}_g=b\ll d$. If
$\bm{D}_g\in\mathbb{R}^{b\times d}$ is an orthonormal basis for $\bm{V}_g$, then every
$\bm{m}_g\in\mathcal{M}_g$ has coordinates $\bm{z}_g\in\mathbb{R}^b$ such that
$\bm{m}_g=\bm{z}_g\bm{D}_g$. Assuming active factors are nonzero
almost surely, let $S=\{g:\|\bm{z}_g\|_2>0\}=\operatorname{supp}_{\mathcal{G}}(\bm{z})$, so that $\bm{z}_g=\bm{0}$ for $g\notin S$. The noisy version of the data generating process then becomes
\[
    \bm{x}
    =
    \sum_{g\in S}\bm{m}_g+\bm{\varepsilon}
    =
    \sum_{g\in S}\bm{z}_g\bm{D}_g+\bm{\varepsilon}
    =
    \bm{z}\bm{D}+\bm{\varepsilon},
    \qquad
    \bm{\varepsilon}\sim\mathcal{N}(\bm{0},\sigma^2\bm{I}_d).
\]
 We can think of each block coordinate $\bm{z}_g$ drawn independently from some prior, such that $\bm{x}$ then arises as the sum of the contributions of (the embeddings back into the $d$-dimensional space of) each coordinate, corrupted by Gaussian noise. Our main result in this section clarifies the nature of the prior, such that our method emerges as the MAP estimate corresponding to this prior: block sparsity supplies the value $\bm{z}$ that maximize a suitable posterior $p(\bm{z}|\bm{x})$.
\begin{proposition}[Block sparsity is the MAP-matched prior]
Let $\bm{x}$ be drawn from the noisy block-linear model
$\bm{x}=\bm{z}\bm{D}+\bm{\varepsilon}$ with
$\bm{\varepsilon}\sim\mathcal{N}(\bm{0},\sigma^2\bm{I}_d)$, such that each block coordinate $\bm{z}_g$ is drawn independently from a block prior
\[
    p(\bm{z}_g)
    =
    (1-\pi)\,\delta_{\bm{0}}
    +
    \pi\,\mathcal{U}_{\mathcal{B}_R},
    \qquad
    \mathcal{B}_R=\{\bm{u}\in\mathbb{R}^b:\|\bm{u}\|_2\le R\},
\] where $\delta_{\bm{0}}$ is the point mass on $\bm{0}$ and $\mathcal{U}_{\mathcal{B}_R}$ is the uniform distribution on the (closed) ball $\mathcal{B}_R$ of radius $R$. Let $\Sigma = \{\bm{z}: \|\bm{z}_g\|_2\le R$ for all $g \leq G\}$ be the support of $p$. Then the MAP estimate is %
\[
    \hat{\bm{z}}
    =
    \argmin_{\bm{z} \in \Sigma}
    \frac{1}{2}\|\bm{x}-\bm{z}\bm{D}\|_2^2
    +
    \lambda\|\bm{z}\|_{2,0},
    \qquad
    \lambda
    =
    \sigma^2
    \log\!\left(
        \frac{1-\pi}{\pi}
        \operatorname{vol}(\mathcal{B}_R)
    \right). \tag{\ref{eq:block_l0}}
\] 
\end{proposition}

\begin{proof} We want to maximize $\log p(\bm{z}|\bm{x}) \propto  \log p(\bm{z}) + \log p(\bm{x}|\bm{z})$. Concerning the prior $\log p(\bm{z}) = \sum_{g} \log p(\bm{z}_g)$, we can consider the two cases for each $p(\bm{z}_g)$, depending on whether $\bm{z}_g = \bm{0}$ (i.e., whether $g \notin S$). If $\bm{z}_g = \bm{0}$, then $\log p(\bm{z}_g)$ contributes $-\log(1-\pi)$ to the sum. If $\bm{z}_g \neq \bm{0}$ (and also $\|\bm{z}_g\|_2\le R$), then it contributes $\log \pi - \log \operatorname{vol}(\mathcal{B}_R)$. So this sum becomes \begin{eqnarray*}
    \log p(\bm{z}) & = & \big(G - \|\bm{z}\|_{2,0}\big)\log(1-\pi) + \|\bm{z}\|_{2,0}\big(\log \pi - \log \operatorname{vol}(\mathcal{B}_R)\big) \\
    & = & G \log(1-\pi) + \|\bm{z}\|_{2,0}\log\!\left( \frac{\pi}{(1-\pi)\operatorname{vol}(\mathcal{B}_R)}  \right).
\end{eqnarray*}
The likelihood is Gaussian, so the log likelihood is
\begin{eqnarray*}
    \log p(\bm{x}|\bm{z}) & = & -\frac{d}{2} \log 2\pi\sigma^2 - \frac{1}{2\sigma^2} \|\bm{x}-\bm{z}\bm{D}\|_2^2. 
\end{eqnarray*}
Dropping the two first constant terms, which do not depend on $\bm{z}$, the negative log posterior becomes
\begin{eqnarray*}
    -\log p(\bm{z}|\bm{x}) & = & \frac{1}{2\sigma^2} \|\bm{x}-\bm{z}\bm{D}\|_2^2 - \|\bm{z}\|_{2,0}\log\!\left( \frac{\pi}{(1-\pi)\operatorname{vol}(\mathcal{B}_R)}  \right) \\
    & = & \frac{1}{2\sigma^2} \|\bm{x}-\bm{z}\bm{D}\|_2^2 + \log\!\left( \frac{1-\pi}{\pi}\operatorname{vol}(\mathcal{B}_R)  \right) \|\bm{z}\|_{2,0}.
\end{eqnarray*} Multiplying by $\sigma^2$ does not change the minimum, so Eq. \ref{eq:block_l0} follows immediately.

\end{proof}

Now, few remarks on this result.
First, one could show that the choice of slab is not load-bearing. A Gaussian slab gives the same selection term plus a ridge penalty on active blocks~\citep{mitchell1988bayesian}, and a Laplace slab on the block norm gives the group lasso penalty $\|\bm{z}\|_{2,1}$ as exact MAP~\citep{casella2010penalized}. The flat slab is simply the cleanest expression of the principle, and the family of penalties it generates is studied in the structured sparsity literature~\citep{yuan2006model,bach2012structured}.
Second, the threshold $\lambda$ grows with the block dimension. The volume of the ball scales as $R^b$, sohigher thresholds are more costly: to put it otherwise, a feature is granted $b$ dimensions only if it earns them in reconstruction.
Third, the $2$-norm inside the blocks is forced and can be interpreted as the signature of activity. As a manifold has no canonical chart: rotating the chart, $\bm{z}_g \mapsto \bm{z}_g\bm{Q}^\top$ and $\bm{D}_g \mapsto \bm{Q}\bm{D}_g$ for $\bm{Q} \in O(b)$, leaves the model unchanged, so a penalty intrinsic to the model must be blind to this rotation. The orbits of $O(b)$ acting on $\mathbb{R}^b$ are the spheres of constant $\ell_2$ norm, so a separable penalty is rotation-invariant exactly when it depends on each block through $\|\bm{z}_g\|_2$. Both $\|\cdot\|_{2,0}$ and $\|\cdot\|_{2,1}$ qualify. 

It is also worth spelling out what happens when the prior is mismatched, since it retrodicts a known phenomenon. Under the model, an active factor occupies $k$ coordinates, so a flat $\ell_0$ budget charges $b$ per faithfully represented factor. A coder can instead ``pay'' by approximating a local neighborhood of the manifold with a single direction, and as the budget tightens, the optimal flat-sparse solution covers each factor with many rank-one atoms, one per region. The tiling of curve detectors in InceptionV1~\citep{cammarata2020curve,gorton2024missing} and the dilution analysis of~\citet{bhalla2026sparse,lubana2025priors} are, under this reading, the predictable optimum of a mismatched prior.

Finally, Eq.~\ref{eq:block_l0} is combinatorial and NP-hard in general~\citep{eldar2009block}, and the structured sparsity literature offers various method to make the optimization tractable. For example, one can relax the penalty to its tightest convex surrogate, the group lasso norm~\citep{yuan2006model}; one can select blocks greedily by residual correlation, as in Block-OMP~\citep{eldar2009block}; or one can keep the sparsity as a constraint and project, retaining the top blocks by energy, in the spirit of iterative hard thresholding.

\section{Downstream performances}

\subsection{Quantization and task ceiling (Figure~\ref{fig:pareto})}
In this experiment, we measure the effect of reconstruction error on task performance. We uniformly quantize the normalized patch activations per patch, $\bm{x_q} = \Delta\,\mathrm{round}(\bm{x}/\Delta)$ and sweep $\Delta$ over 13 levels from $0$ to $7$. Then, we feed the quantized features to a linear probe trained once on the clean activations and held frozen (`Basic Probe') and record the task metric relative to its unquantized value. The corruption is measured by the centered reconstruction fidelity $R^2(\bm{x},\bm{x_{q}})=1-\lVert \bm{x}-\bm{x_q}\rVert^2/\lVert \bm{x}-\bar{\bm{x}}\rVert^2$ over the same tokens. Probes report ImageNet-1k top-1, ADE20k mIoU, and NYUv2 $\delta_1$ (the fraction of pixels whose predicted depth is within $25\%$ of ground truth).

\subsection{Single-concept F1 (Figure~\ref{fig:f1_tv_dirichlet}, left)}
For each ImageNet-1k class we find the single concept that best detects it. Per patch a block emits a $k$-vector of code values, which we read out along a unit direction $\bm{v}$ in the block's $k$-dimensional subspace (for the vanilla $k{=}1$ SAE this is just the signed code). The image-level score of a (block, $\bm{v}$) pair is the maximum of this projection over the 196 patches. For each (class, block) we fit $\bm{v}$ by $\ell_2$-regularized logistic regression on the block's firing patches labeled by image class, set the threshold that maximizes F1, and keep the block with the highest F1. Selection uses a 200k-image train subsample.

\subsection{Concept-map smoothness (Figure~\ref{fig:f1_tv_dirichlet}, center and right)}
\label{app:tv}
For each image and concept we form the $14\times14$ map of its per-patch activation (block: group $\ell_2$ norm; vanilla: $|\text{code}|$) and keep the 8 highest-energy concepts per image, unit-$\ell_2$ normalizing each map. We report total variation (sum of absolute 4-neighbor differences) and Dirichlet energy (sum of squared 4-neighbor differences), averaged over the 8 maps.

\subsection{Linear probing and cosine probe recovery (Figures~\ref{fig:downstream_tasks} and~\ref{fig:probe_recovery})}
With the backbone and featurizer frozen, we fit linear probes on the extracted codes for ImageNet-1k classification (max-abs pooled to image level), ADE20k segmentation, and NYUv2 depth (both per patch). We sweep over an SGD/AdamW learning-rate and weight-decay grid, picking the best performing validation probe for each condition. The basic probe baseline is the same probe fit directly on the raw activations, and we optionally concatenate the CLS token for classification and depth. For cosine probe recovery, we take each trained raw-activation probe direction and measure the cosine between it and its projection onto the span of its best-aligned block's $k$ decoder atoms. We report the median and IQR over probe directions against a random-direction floor and a real-activation ceiling.

\paragraph{Additional Results}

We train BSFs of varying block dimension $k$ on ImageNet-1k activations, freeze both the backbone and the featurizer, and fit linear probes on the extracted codes for three tasks, classification on ImageNet-1k, semantic segmentation on ADE20k, and monocular depth on NYUv2, comparing against a vanilla SAE (the $k{=}1$ BSF) and against a probe read directly off the raw DINOv3 activations (the basic probe). The codes improve on the basic probe for classification and approach it without surpassing it for segmentation and depth, and performance rises with the block sparsity $\ell$ on all three tasks (Figure~\ref{fig:downstream_tasks}).

\begin{figure}[h]
    \centering
    \includegraphics[width=0.95\linewidth]{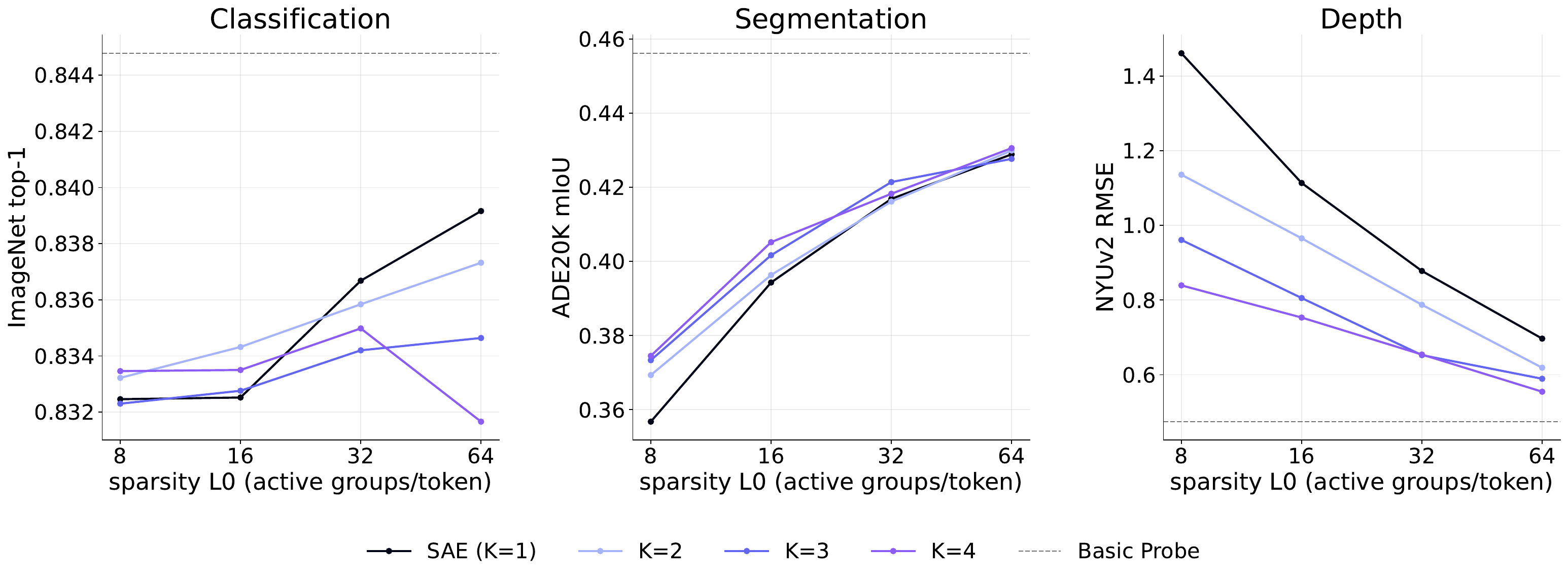}
    \caption{\textbf{Linear Probe Evaluation}. We train linear probes on codes extracted by Vanilla SAEs and BSFs (Block-SAEs) on ImageNet-1k (classification), ADE20k (semantic segmentation) and NYUv2 (monocular depth estimation). We report validation set performance on each dataset.}
    \label{fig:downstream_tasks}
\end{figure}

\begin{figure}[h]
    \centering
    \includegraphics[width=0.95\linewidth]{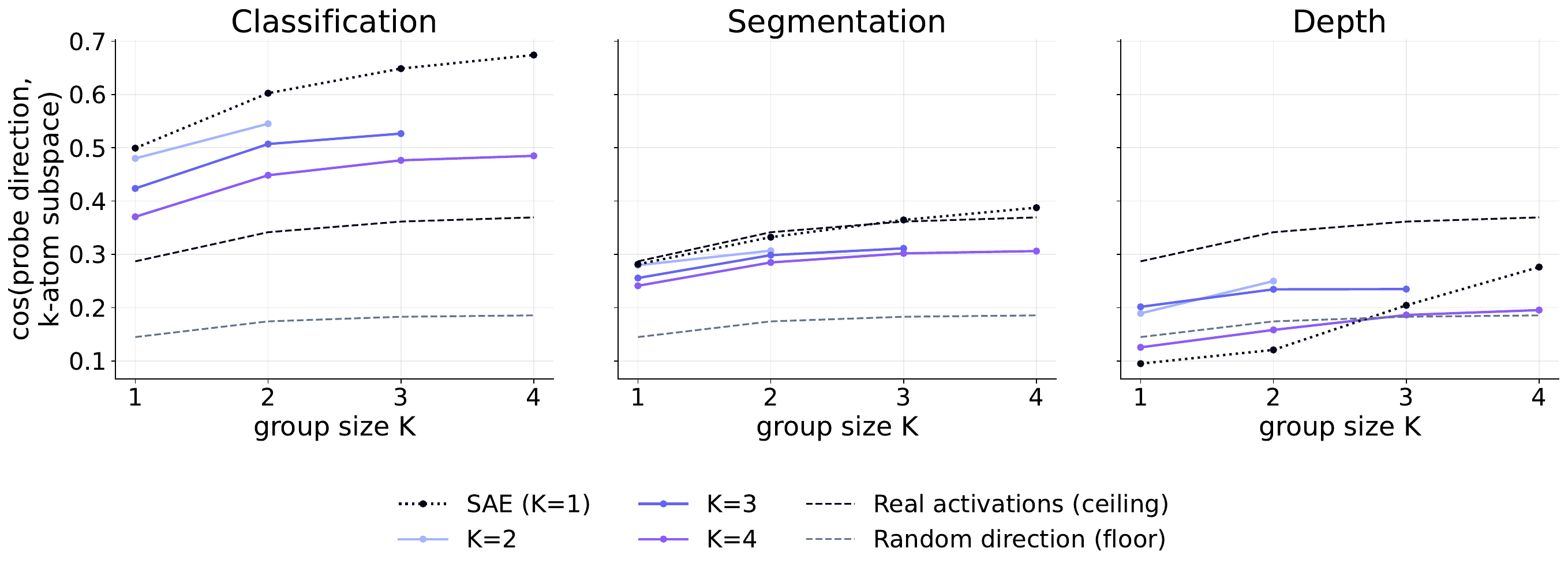}
    \caption{\textbf{Probe Recovery}. We treat linear probes as activations and measure how well different SAEs can reconstruct them.}
    \label{fig:probe_recovery}
\end{figure}

\begin{figure}[h]
    \centering
    \includegraphics[width=0.95\linewidth]{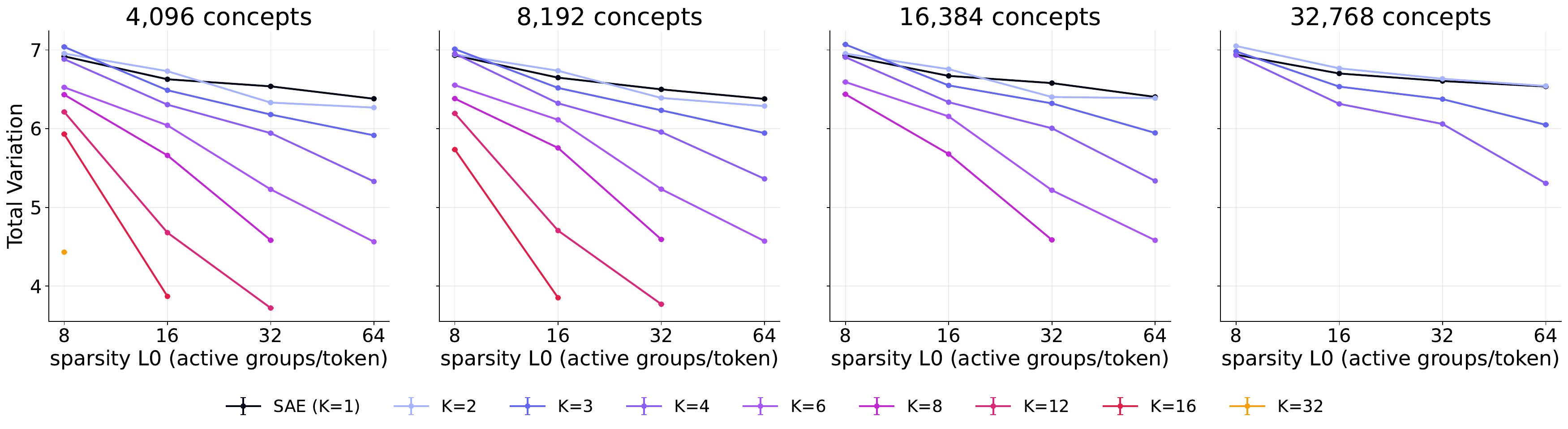}
    \caption{\textbf{BSFs learn more spatially coherent concepts}. We show total variation as a function of number of concepts predicted for a sweep of 96 SAEs.}
    \label{fig:smoothness_sweep}
\end{figure}

\section{Toy model of Manifold Superposition}
\label{app:toma}

The toy model is simple: each observation is a sparse sum of low-dimensional manifolds embedded into $\mathbb{R}^d$ by random orthonormal maps, $\bm{x} = \sum_{i \in S} \tilde{\bm{\gamma}}_i(\theta_i)\,\bm{V}_i + \bm{\varepsilon}$ with $|S| = L_0$, so that the question we put to a featurizer is whether it recovers the individual factors from their mixture.

\paragraph{Manifold zoo.}
Table~\ref{tab:toma_zoo} collects the manifold types we use, giving for each its intrinsic dimension $d_i$, the dimension $\blocksize_i$ of the ambient subspace its embedding occupies, and the parametric map $\bm{\gamma}_i$. The gap between $d_i$ and $\blocksize_i$ is the distinction that governs the construction throughout, since a circle is parameterized by a single angle, $d_i{=}1$, yet its embedding $(\cos\theta, \sin\theta)$ spans a two-dimensional subspace, $\blocksize_i{=}2$, so that a block must hold two directions to carry it. We instantiate $M = 128$ factors at ambient dimension $d = 128$, half of them one-dimensional concept atoms in the classical dictionary-learning sense and half curved manifolds drawn in turn from the seven curved types.

\begin{table}[h]
\centering
\footnotesize
\caption{Manifold zoo used in the controlled toy.}
\label{tab:toma_zoo}
\setlength{\tabcolsep}{3pt}
\begin{tabular}{lccl}
\toprule
\textbf{Type} & $d_i$ & $\blocksize_i$ & \textbf{Embedding} $\bm{\gamma}_i(\theta) \in \mathbb{R}^{\blocksize_i}$ \\
\midrule
Concept (segment) & 1 & 1 & $(t)$ \\
Circle    & 1 & 2 & $(\cos\theta,\; \sin\theta)$ \\
Flat disk & 2 & 2 & $(r\cos\theta,\; r\sin\theta),\quad r \sim \sqrt{\mathcal{U}(0,1)}$ \\
Sphere    & 2 & 3 & $(\sin\phi\cos\theta,\; \sin\phi\sin\theta,\; \cos\phi)$ \\
Torus     & 2 & 3 & $((R{+}r\cos\phi)\cos\theta,\; (R{+}r\cos\phi)\sin\theta,\; r\sin\phi)$ \\
M\"obius  & 2 & 3 & $((1{+}t\cos\tfrac{\phi}{2})\cos\phi,\; (1{+}t\cos\tfrac{\phi}{2})\sin\phi,\; t\sin\tfrac{\phi}{2})$ \\
Swiss roll& 2 & 3 & $(\theta\cos\theta,\; h,\; \theta\sin\theta)$ \\
Helix     & 1 & 3 & $(\cos\theta,\; \sin\theta,\; \alpha\theta)$ \\
\bottomrule
\end{tabular}
\end{table}

\paragraph{Normalization.}
So that every factor weighs equally in the reconstruction loss, we center and isotropically rescale each instance at construction, since without this correction the large-coordinate manifolds, such as the Swiss roll whose coordinates run to $\theta \sim 4.5\pi$, would consume the featurizer's capacity while the small ones would be read as noise. For each instance $i$ we draw a calibration sample of $50{,}000$ points from the raw embedding $\bm{\gamma}_i$, compute the mean $\bm{\mu}_i$ and the RMS norm of the centered samples $\sigma_i = \sqrt{\mathbb{E}\|\bm{\gamma}_i(\theta) - \bm{\mu}_i\|^2}$, and set $\tilde{\bm{\gamma}}_i(\theta) = (\bm{\gamma}_i(\theta) - \bm{\mu}_i)/\sigma_i$. Being a translation composed with an isotropic rescaling, this leaves angles, relative distances, curvature ratios, and topology untouched and fixes the RMS norm of every instance to one in local coordinates regardless of type.

For each instance we draw a random orthonormal matrix $\bm{V}_i \in \mathbb{R}^{\blocksize_i \times d}$ as the transposed $\bm{Q}$ factor of the QR decomposition of a $d \times \blocksize_i$ Gaussian, so that the embedding preserves norm, $\|\bm{z}\bm{V}_i\|_2 = \|\bm{z}\|_2$, with bias offsets set to zero throughout.

\paragraph{Sparse mixture sampling.}
Observations follow
\begin{equation*}
    \bm{x} = \sum_{i \in S} \tilde{\bm{\gamma}}_i(\theta_i)\,\bm{V}_i + \bm{\varepsilon}, \qquad |S| = L_0 = 4,
\end{equation*}
where the active set $S$ is drawn uniformly without replacement from the $M$ instances and the intrinsic coordinates $\theta_i$ uniformly on each manifold, and we take clean activations, $\bm{\varepsilon} = \bm{0}$, so that recovery reflects geometric organization rather than denoising. We generate $N = 3{\times}10^{5}$ training samples and a held-out evaluation set of $10^{5}$ samples at $L_0{=}4$ under a separate seed, retaining the per-manifold contributions $\bm{m}_i = \tilde{\bm{\gamma}}_i(\theta_i)\bm{V}_i$ and the active masks, so that evaluation measures recovery on in-distribution superpositions whose ground truth is known. Activations are scaled so that $\mathbb{E}\|\bm{x}\|_2^2 = 1$.

\paragraph{Featurizer training.}
We train the three BSFs of Eq.~\ref{eq:bsf} together with the baselines, a TopK SAE that is the $\blocksize{=}1$ Vanilla BSF and the two concurrent subspace featurizers discussed below, using the single-forward-pass encoders of Appendix~\ref{app:implementations} and Adam over a few hundred epochs. Each featurizer is reported at its best configuration over a sweep of block sparsity $\topk$, dictionary width $G$, and block dimension $\blocksize$, the Vanilla BSF for instance at $\blocksize{=}4$, $\topk{=}4$, $G{=}2M$, and no reconstruction noise or auxiliary loss enters unless the method itself prescribes it.

\paragraph{Per-block recovery ($R^2$).}
The metric tests Def.~\ref{def:amm} directly, asking for each ground-truth factor how well a single recovered block reconstructs that factor's contribution. We encode the evaluation set to codes $\{\bm{z}^{(j)}\}$ and, for each instance $i$, restrict to the rows on which $i$ is active using the ground-truth masks, which gives codes $\bm{Z}_i$ and true contributions $\bm{M}_i$; we then match instance $i$ to the block $g^\star(i)$ whose firing best predicts its active mask and reconstruct the contribution from that block alone, $\hat{\bm{M}}_i = \bm{Z}_i^{(g^\star)} \bm{D}$, where $\bm{Z}_i^{(g^\star)}$ zeroes all blocks but $g^\star(i)$. The per-block $R^2$ is
\begin{equation*}
    R^2(i) = 1 - \frac{\sum_j \|\bm{m}_i^{(j)} - \hat{\bm{m}}_i^{(j)}\|^2}{\sum_j \|\bm{m}_i^{(j)} - \bar{\bm{m}}_i\|^2},
\end{equation*}
with $\bar{\bm{m}}_i$ the mean true contribution, averaged over factors that are active sufficiently often. For the directional SAE the matched block is taken to be its $\blocksize_i$ best-fitting atoms, the restricted-$R^2$ subspace-capture reading, so that each featurizer is granted exactly the ambient dimension its factor requires. An oracle that reconstructs each $\bm{m}_i$ from the true active subspaces by least squares reaches $R^2 \approx 0.99$, which we report as the recovery ceiling, and Figure~\ref{fig:toy_leaderboard} reports the mean per-block $R^2$.

\begin{figure}[ht]
    \centering
    \includegraphics[width=\linewidth]{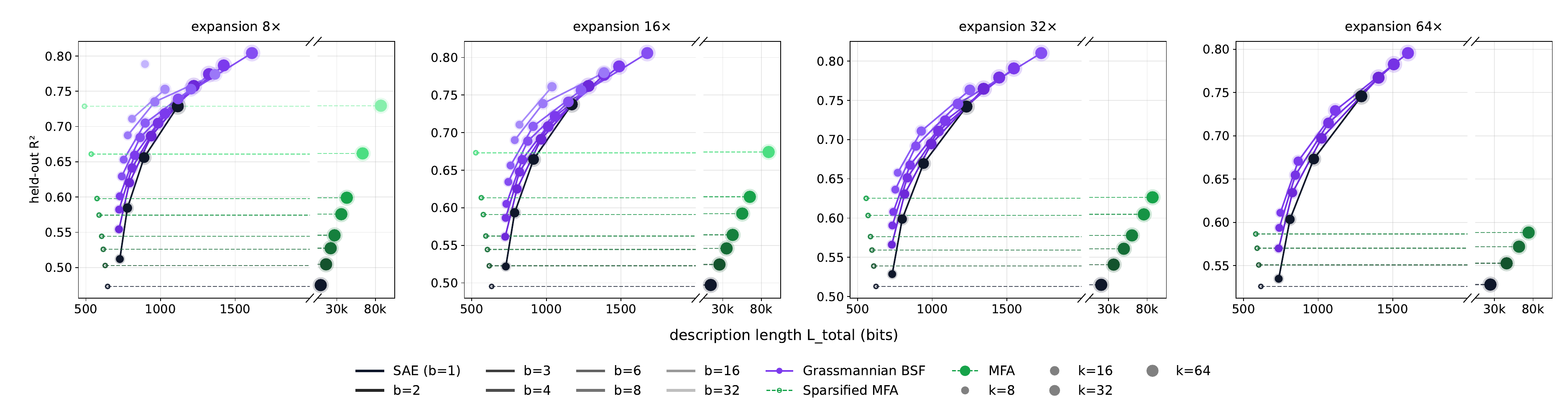}
    \caption{\textbf{Minimal description length comparison between Grassmannian BSFs and MFAs}. MFAs are naturally dense, resulting in high description lengths during inference. Hard top-k thresholding can be applied to MFAs to induce sparsity, as shown by the open circles (sparsity of k=1), significantly reducing the MDL without impacting $R^2$ significantly. }
    \label{fig:mfa_mdl}
\end{figure}

\paragraph{Adapting SMixAE and MFA.}
SMixAE encodes through a rectified expert space and a per-expert low-dimensional bottleneck, mapping the input with a LearkyReLU after an affine projection and reshaping the result into $K$ experts, then it projecting each to a $\blocksize$-dimensional bottleneck, select experts by a batch-level Top-$\topk$ on the bottleneck norm, and decoding through $\bm{W}^{\mathrm{lat}}_e$ and a shared $\bm{W}_{\mathrm{dec}}$. As published it reaches per-block $R^2 \approx 0.71$ on the toy, and two changes that align it with the block-sparse prior lift it to $R^2 \approx 0.85$: removing the LeakyRelu so that the code is signed and replacing the batch routing with the per-sample block projection $\bm{\Pi}_{\topk}$ on the bottleneck norm, which is the Vanilla-BSF selection. 

For MFA~\citep{shafran2026directions}, it fails for a different reason, its native convex combination (the responsability coefficients) allow MFA to reach only per-block $R^2 \approx 0.38$ and a global reconstruction that saturates near $0.47$, because a mixture assigns each activation to one component and so cannot account for an additive sum, as the next paragraph makes precise. 
Re-using its learned per-component subspaces $\{\bm{W}_c\}$ as a block dictionary and decoding additively, by selecting the $\topk$ best-fitting subspaces per activation and solving a single joint least-squares over their union, optionally followed by a short block-Top-$\topk$ fine-tune, raises recovery to $R^2 \approx 0.88$. This additive decode abandons the model's own one-component generative assumption and is therefore a decode-time heuristic rather than a coherent featurizer, included only to establish that the learned subspaces are themselves usable once they are read additively.

\paragraph{The implicit data-generating process of MFA.}
It is worth making explicit why MFA is mismatched to Def.~\ref{def:amm}, as a mixture of factor generates an activation by drawing a single component and then a Gaussian within it,
\begin{equation*}
c \sim \mathrm{Cat}(\bm{\pi}), \qquad \bm{z} \sim \mathcal{N}(\bm{0}, \bm{I}_q), \qquad \bm{x} = \bm{\mu}_c + \bm{W}_c \bm{z} + \bm{\varepsilon}, \qquad \bm{\varepsilon} \sim \mathcal{N}(\bm{0}, \bm{\Psi}_c),
\end{equation*}
with $\bm{W}_c \in \mathbb{R}^{q \times d}$ a low-rank loading and $\bm{\Psi}_c$ a diagonal noise, so that the marginal of component $c$ is $\mathcal{N}(\bm{\mu}_c, \bm{W}_c^\top \bm{W}_c + \bm{\Psi}_c)$. Read in the language of Sec.~\ref{subsec:bsf}, this is a block model with exactly one active block: the categorical $c$ selects a single subspace $\mathcal{V}_c = \operatorname{span}(\bm{W}_c)$, the code $\bm{z}$ places the activation within it, and the diagonal $\bm{\Psi}_c$ absorbs the off-subspace residual. It is therefore the one-sparse, $|S|{=}1$ special case of Eq.~\ref{eq:linearized_dgp}, with the spike-and-slab prior collapsed to the statement that exactly one factor is on.

The mismatch is now visible in a single number. Def.~\ref{def:amm} admits $|S| > 1$, so an activation is a sum of several factors, whereas the mixture admits only $|S|{=}1$ and so represents a choice among factors. This is the difference between co-presence and selection: when the fitted mixture is uncertain, its responsibilities spread over the components as a convex weighting on the simplex, but that weighting still encodes uncertainty about which one factor produced the activation, and a convex combination of single factors can never be their Minkowski sum. A genuinely additive probabilistic treatment of the factor-analyzer family, one whose generative process places $|S| > 1$ factors on at once, is a natural and promising direction that lies outside the scope of this work.

\end{document}